\documentclass[nohyperref,twocolumn]{article}

\usepackage{microtype}
\usepackage{graphicx}
\usepackage{subfigure}
\usepackage{booktabs} 

\usepackage{hyperref}

\usepackage{amsmath}
\usepackage{amssymb}
\usepackage{mathtools}
\usepackage{amsthm}

\usepackage[capitalize,noabbrev]{cleveref}

\theoremstyle{plain}
\newtheorem{theorem}{Theorem}[section]
\newtheorem{proposition}[theorem]{Proposition}
\newtheorem{lemma}[theorem]{Lemma}

\theoremstyle{definition}
\newtheorem{definition}[theorem]{Definition}
\newtheorem{assumption}[theorem]{Assumption}
\theoremstyle{remark}

\usepackage{mkolar_definitions}
\newcommand\independent{\protect\mathpalette{\protect\independenT}{\perp}}
\def\independenT#1#2{\mathrel{\rlap{$#1#2$}\mkern2mu{#1#2}}}
 
\usepackage{lipsum}
\usepackage{tikz}
\usetikzlibrary{bayesnet}
\usetikzlibrary{arrows}
\usepackage{sidecap}
\usepackage[export]{adjustbox}
\usepackage{natbib} 
\usepackage{algorithm}
\usepackage{algorithmic}

\usepackage[textsize=tiny]{todonotes}

\begin{document}

\title{Correcting Confounding via Random Selection of \\ Background Variables}

\author{You-Lin Chen$^1$, Lenon Minorics$^2$,  Dominik Janzing$^2$\\
{\small 1) Amazon Research Seattle, USA } \\
{\small 2) Amazon Research T\"ubingen, Germany }\\
{\small \{cyoulin, minorics,janzind\}@amazon.com} }

\maketitle

\begin{abstract}
	We propose a method to distinguish causal influence from hidden confounding in the following scenario: given a target variable $Y$, potential causal drivers $X$, and a large number of background features, we propose a novel criterion for identifying causal relationship based on the stability of regression coefficients of $X$ on $Y$ with respect to selecting different background features. To this end, we propose a statistic ${\cal V}$ measuring the coefficient's variability. We prove, subject to a symmetry assumption for the background influence,  that ${\cal V}$ converges to zero if and only if $X$ contains no causal drivers. In experiments with simulated data, the method outperforms state of the art algorithms. Further, we report encouraging results for real-world data. Our approach aligns with the general belief that causal insights admit better generalization of statistical associations across environments, and justifies similar existing heuristic approaches from the literature. 
\end{abstract}

\section{Introduction} \label{sec:intro}

Understanding causal relations is crucial for many if not all scientific disciplines. Data-driven inference of causal relations is commonly based
on observing statistical dependences. According to \citet{reichenbach1956direction}, every statistical dependence between two random variables $X$ and $Y$ is due to $X$ influencing $Y$, $Y$ influencing $X$, or a common cause influencing both $X$ and $Y$. To distinguish between these alternatives
is challenging and identification of causal directions from passive observations alone requires strong assumptions on the data generating process, see e.g., Proposition 4.1 in
\citet{peters2017}. If $X$ and $Y$ are observed together with additional variables,
Markov condition and faithfulness admit identification up to Markov equivalence classes \cite{spirtes2000causation}. 
In linear acyclic models with non-Gaussian errors, cause and effect can be identified by independent component analysis \cite{shimizu2006linear}. In Gaussian linear structural equation models, identifiability of the exact causal structure (rather than only the Markov equivalence class) can be proved when error terms have equal variances \cite{peters2014identifiability, chen2019causal}. In the setting of a non-linear data-generating process with additive noise, a causal system becomes identifiable under mild conditions~\cite{NIPS2008_f7664060}.

In many practical applications one of the causal directions, say $Y\to X$,  can be excluded (e.g. the patient's health condition after a medical treatment cannot be the cause of the latter) due to time order, and the remaining challenge consists in inferring to what extent the dependences between $X$ and $Y$ are due to
the influence of $X$ on $Y$ and which part is due to a common cause of both. If a complete list of common causes is given, the problem boils down to
a statistical problem whose challenge comes from the large number of confounders \cite{Imbens2015, chernozhukov2018double, shalit2017estimating, raj2020causal}.
Our work considers the same scenario, but does not assume that the full list of confounders is known. This way, identification of the effect of $X$ on $Y$ 
requires additional assumptions and methodology beyond statistics. 
Here we assume there are $m$ datasets $\{(Y, X, E_i)\}_{i=1}^m$, possibly from different environments, where $\{E_i\}_{i=1}^m$ is a set of background features. For example, data $(Y, X, E_i)$ are collected from different countries, and $E_i$ are population and several economic indicators.
For each $(Y, X, E_i)$, we obtain different regression coefficients of $X$. The key insight (provable subject to a symmetry assumption) is that the stability of the regression coefficients with respect to different background covariates $E_i$ is a reasonable criterion for the regression coefficients to be causal. \citet{lu2021causal} used a similar heuristic idea to recover a directed gene network in a time series framework but without theoretical justification. This idea is in the spirit of an increasingly popular belief about the following relationship:

\noindent
relation causal $\Leftrightarrow$ robust and invariant statistical association

Already \citet{haavelmo1943statistical} argued that causal statements should be robust regarding different perturbations and environments. Recent approaches to causal inference exploit the opposite direction of this relationship. For example, \citet{Peters2016causal, buhlmann2020invariance,heinze2018invariant, bengio2020meta} argued that the predictions from a causal model are more robust under different environments, while predictions from a non-causal model generalize poorly. They further showed that the causal relationship could be identified given samples from different experimental settings. See, however, also critical remarks in \citet{rosenfeld2021risks, kamath2021does}.

This work proposes a method to distinguish causal influence from hidden confounding based on the stability of the regression coefficients. Our scenario differs from other theoretical results exploiting the relation between invariance and causality in the literature as follows. First, our environments differ with respect to different background features $E_i$, while the above-mentioned literature typically considers the same variables with distributional shifts. Second, we explicitly allow for unobserved confounders and propose a method to correct for confounding effects.
Our method, however, requires an assumption for the prior distribution of the true causal coefficients of the background variables. While we are aware that our assumption is rather restrictive and debatable, correcting for unobserved confounding is an extremely hard problem which is only solvable subject to strong assumptions. As valid for many methods in ML, and causality in particular, the essential question is whether the method shows some robustness regarding violations. Our empirical study with real datasets in Section~\ref{sec:experiments} suggests some robustness in this regard.  Moreover, similarly strong assumptions relying on symmetries (e.g.~\citet{janzing2019causal, janzing2018detecting}) or a confounder that affect multiple measured features 
\cite{wang2019blessings} are commonly used in confounder correction. Our contributions can be summarized as follows:

\begin{enumerate}
	\item We propose a procedure of detecting causal relations based on the stability of regression coefficients with respect to different background features. We examine an example of structural equation models in detail and provide sufficient assumptions under which coefficients' stability can be used to detect the true causal drivers even in the presence of unobserved confounding.
	
	\item We propose a test statistic to quantify the stability of regression coefficients and  show that under a mild condition on the strength of the confounding effect, the test statistic is zero if and only if $X$ is the causal driver of $Y$.
	
	\item We derive a permutation test for inferring the causal influence of X on Y based on our model assumption and proposed statistic. 
\end{enumerate}

\section{Related work}

{\bf Stability of regression coefficients:} \citet{leamer1978specification, lu2014robustness, chen2015exogeneity}  related coefficient robustness to unobserved confounding and causal inference.
\citet{oster2017unobservable} built a theory connecting omitted variable bias to coefficient stability and used coefficient's movement to estimate a confounding effect. However, their method required the strong assumption that unobservable and observable confounders are independent, which we do not need.
Likewise, \citet{Imbens_2003} estimated the amount of variation of treatment and variable required to generate significant confounding bias and discuss
plausibility of strong bias based on these insights. However, also these estimations assume observable and unobservable confounders to be independent, and also specific parametric models. 
\citet{cinelli2019making} was based on a similar idea as \citet{Imbens_2003}, with a different model fit method and relaxing some parametric assumptions.  


{\bf Correcting hidden confounding:}
Instrumental variables~\cite{white1982instrumental, muandet2019dual} is a classic approach to correcting for hidden confounding and estimating causal effects. However, the existence of instrumental variables is always debatable. Several recent works proposed new methods utilizing different causal structures and assumptions. \citet{NIPS2019_9419} proposed a constraint-based causal feature selection method for identifying potential causes under the setting that there can also be hidden variables acting as common causes, provided that a cause for each candidate cause could be observed. \citet{wang2019blessings} used unsupervised machine learning to learn the representation of unobserved confounders in multiple-cause settings. \citet{chernozhukov2017lava} assumed the hidden confounding effect is sparse and combined lasso and ridge regression to estimate the bias and the true structure coefficient, and \citet{cevid2018spectral, guo2020doubly} established its asymptotic normality and efficiency in the Gauss-Markov sense. \citet{janzing2018detecting} suggested confounder detection methods that are inspired by the intuition that $P_X$ and the conditional $P_{Y|X}$ do not contain information about each other.

{\bf Conditional independence tests:} Given that the association between $X$ and $Y$ is only confounded by {\it observed} variables, high-dimensional conditional independence (which is, however, hard to test \cite{shah2020hardness}) detects whether $X$ influences $Y$. For linear and parametric methods, \citet{zhang2014confidence, van2014asymptotically} constructed confidence intervals for individual coefficients or linear combinations of some variables and showed asymptotically optimality under certain assumptions. \citet{dezeure2017high} proposed a bootstrap approach to construct confidence intervals. For more nonlinear extension, see \citet{ramsey2014scalable, fan2020projection} for using partial correlation and \citet{fukumizu2008kernel, zhang2012kernel, muandet2020kernel} for kernel-based methods. In particular, \citet{azadkia2021simple} proposed a nonlinear generalization of the partial $R^2$ statistic and showed 
(without any distributional assumption)
that the limit of the proposed statistic is zero if and only if conditional independence holds. \citet{barber2020testing} generalized Goodness-of-fit tests to testing conditional independence.

\section{Notions and Problem Setup}

\begin{figure*}[t]
	\centering
	\resizebox{0.4\textwidth}{!}{
		\begin{tikzpicture}
			\node[obs]                               (Y) {$Y$};
			\node[obs, below= 0.5cm of Y] (X) {$X$};
			
			\node[latent, left =1.5cm of X, yshift=-1.6cm] (Z1){$Z_1$};
			\node[latent, right =1.5cm of X, yshift=-1.6cm] (Z2){$Z_2$};
			
			\node[obs, left =3.8cm of Y, yshift=-0.8cm]  (W1) {$W_1$};
			\node[obs, left =4cm of Y]  (W2) {$W_2$};
			\node[obs, left =3.5cm of Y, yshift=0.7cm]  (W3) {$W_3$};
			
			\node[latent, left = 2.8cm of Y, yshift=1.6cm]  (W4) {$W_4$};
			
			\node[obs, above=1.2cm of Y, xshift=-0.5cm]  (W5) {$W_5$};
			\node[obs, above=1.3cm of Y, xshift=0.4cm]  (W6) {$W_6$};
			\node[obs, above=1.3cm of Y, xshift= 1.4cm]  (W7) {$W_7$};    
			\node[obs, above=1cm of Y, xshift= 2.3cm]  (W8) {$W_8$};
			
			\node[latent, right=3.5 cm of Y, yshift=.9cm]  (W9) {$W_9$};
			
			\node[obs, right=4cm of Y]  (W10) {$W_{10}$};
			\node[obs, right=4.5cm of Y, yshift= -0.8cm]  (W11) {$W_{11}$};
			\node[obs, right=3.8cm of Y, yshift= -1.4cm]  (W12) {$W_{12}$};
			\edge {X} {Y} ;
			\edge {W1, W2,W3,W4} {Y} ;
			\edge {W5,W6,W7,W8} {Y} ;
			\edge {W9, W10,W11,W12} {Y} ;
			\edge {Z1}{W1, W3, W4, W5, X};
			\edge {Z2}{W7, W9, W12, W1, X};
			
			\plate {} {(Y)(W2)(W3)(X.south east)} {$E_1$} ;
			\plate {} {(Y)(W5)(W7)(W8)(X)} {$E_2$} ;
			\plate {} {(X)(Y)(W11)(W12)} {$E_3$} ;
		\end{tikzpicture}
	} \quad \quad
	\resizebox{0.3\textwidth}{!}{
		\begin{tikzpicture}
			\node[obs]                               (Y) {$Y$};
			\node[obs, left=1cm of Y, yshift=0.65cm] (X) {$X$};
			\node[obs, left=1cm of Y, yshift=-0.65cm]  (W) {$W$};
			\node[latent, left=1cm of X, yshift=-0.65cm] (Z) {$Z$};
			
			\path[->,draw]
			(Z) edge node[above] {$\Bb$} (X)
			(Z) edge node[below] {$\Ab$} (W)
			(X) edge node[above] {$\beta$} (Y)
			(W) edge node[below] {$\gamma$} (Y);
		\end{tikzpicture}
	}
	\caption{Left: An example of our scenario is provided. An arrow indicates a causal effect in the arrow's direction, e.g., $ X $ causes $ Y $. $Z$ is a latent variable and introduced for creating a non-trivial relationship between $X$ and $W$. Suppose no dataset contains all variables in $Y, X, W$. Instead, we have three datasets: $\lbrace Y, X, W_{S_1} \rbrace$, $\lbrace Y, X, W_{S_2} \rbrace$, $\lbrace Y, X, W_{S_3} \rbrace$ where $S_1=\{1,2,3\}, S_2=\{5,6,7,8\}, S_3=\{10,11,12\}$. Note that $\lbrace W_4, W_9\rbrace$ are not contained in any dataset. The existence of the unobserved confounders $\lbrace Z_1, Z_2, W_4, W_9\rbrace$ implies conditional independence tests cannot directly solve our scenario. Right: This figure visualizes the model \eqref{eq:se1} where the latent variables $Z$ are independent sources that influence $X$ and $W$ at the same time.} \label{fig:story}
\end{figure*}
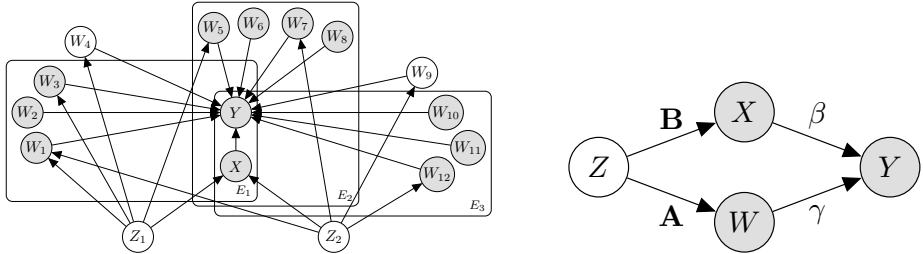

This section presents our notions, model assumptions, and the definition of causality via structural equation models.

{\bf Notations.} Let $S$ be an ordered subset of $\lbrace 1,\dots,q \rbrace$ and $S^c=\lbrace 1,\dots,q\rbrace -S$. $|S|$ denotes the cardinality of $S$. Given a vector $v=(v_1, \dots, v_q)^\top$ and a matrix $A=(a_1,\dots, a_d)^\top$, denote $v_S=(v_{i})_{i \in S}^\top$ as the sub-vector of $v$ with respect to $S$ and $\Ab_{S} = (a_i)_{i \in S}^\top$ as the sub-matrix of $A$. $\Ncal(\mu, \Sigmab)$ denotes the multivariate normal distribution with the mean vector $\mu$ and the covariance matrix $\Sigmab$. $X_1 \sim X_2$ means $X_1, X_2$ are equal in law. $\Ib$ and $\Ib_d$ are the identity matrix of unspecified size and of size $d$, respectively. Similarly, $\zero$ is zero matrix of unspecified size. Given $a\in \RR^d$, $\diag(a) \in \RR^{d \times d}$ is a matrix in which the entries outside the main diagonal are all zero and diagonal elements are equal to $a$. $\|\cdot\|$ denotes the Euclidean norm. For a real-valued function $f(x)$ and $g(x)$, we write $f(x)=o(g(x))$ and $f(x)=O(g(x))$ if $\lim_{x \rightarrow \infty} f(x)/g(x) = 0$ and there exist a constant $c$ such that $\lim_{x \rightarrow \infty} f(x)/g(x) = c$, respectively. $X \stackrel{p}{\rightarrow} Y$ means $X$ converges to $Y$ in probability.

{\bf Problem Setup.} 
To understand why regression coefficients' stability is a reasonable criterion for detecting causal relations, we utilize an example of structural equation models (SEMs) to examine this idea in detail. To be specific, we consider the following linear SEM.
\begin{equation}\label{eq:se1}
	\begin{split}
		Z &\leftarrow N_z \in \RR^r, \\
		W &\leftarrow \Ab Z + N_w \in \RR^q, \\
		X &\leftarrow \Bb Z  +N_x \in \RR^d, \\ 
		Y &\leftarrow  \beta^\top X + \gamma^\top W + N_y \in \RR,
	\end{split}
\end{equation}
where $\Ab = (a_1, \dots, a_r) \in \RR^{q \times r}, \Bb = (b_1, \dots, b_r) \in \RR^{d \times r}$, $N_\square,$ are independent noises such that $\Cov(N_\square)=\diag (\sigma^2_\square)$ for $\square = x,y, w, z$. In SEM \eqref{eq:se1}, $Y$ is an outcome of interest, $X $ represents the potential cause, $W$ are the auxiliary background features, $Z$ is a hidden confounder, and $d, q, r$ are their dimensionality, respectively. The corresponding directed acyclic graph is shown in Figure~\ref{fig:story} right.

The causal interpretation of \eqref{eq:se1} implies that $X$ with $\|\beta\|\neq 0$ is called a {\it cause} of $Y$. Our work considers two crucial tasks: 

\noindent
{\it (i) qualitative causal analysis:} identify the cause $X$, that is, determine whether $\|\beta\| \neq 0$ \footnote{or, more realistically, $\|\beta\|$ should be above a certain threshold} 

\noindent
{\it (ii) quantitative causal analysis:} estimate $\beta$. 

Throughout, we assume that the data $\{(y_i, x_i, w_i, n_{y,i})\}_{i=1}^\ell$ consist of $\ell$ independent and identically distributed (i.i.d.) samples that are drawn from the SEM \eqref{eq:se1} and write 
$\Yb=(y_1, \dots, y_\ell)^\top \in \RR^\ell$, $\Xb = (x_1, \dots, x_\ell)^\top \in \RR^{\ell \times d}$, $\Wb=(w_1,\dots, w_\ell)^\top \in \RR^{\ell \times q}$ and $\Nb_y = (n_{y,1}, \dots, n_{y,\ell})^\top \in \RR^\ell $.

Our scenario supposes that there are $m$ nonempty subsets $S_1, \dots, S_m$ of $\{1, \dots, q\}$ such that the background feature $E_j$ correspond to possibly different selections $W_{S_j}$, i.e. $E_j=W_{S_j}$. It is worth emphasizing that our idea of detecting causes does not highly depend on the correlation structure between $X$ and $Y$. In fact, our idea only requires the following high-level causal assumptions.
\begin{assumption} \label{amp:causal}
	(i) All confounding paths between $X$ and $Y$ are blocked by $W$
	(ii) $W$ does not contain an effect of $Y$ and its descendants.
	(iii) $X$ does not influence $W$.  
	(iv) The structural equations are invariant across different $E_j$.
	(v) There is no systematic bias caused by hidden confounders.
\end{assumption}

It can be shown that SEM~\eqref{eq:se1} satisfies (i) - (iv) in our causal assumption~\ref{amp:causal}.
Assumption (i) and (ii) ensure that interventional probabilities $P(Y | do(X=x))$ \cite{pearl2009causality} can be computed from the joint distribution of $Y, X, W$. For the simple case of a linear regression model, the linear influence of $X$ on $Y$ is given by regressing $Y$ on $X, W$, and ignoring the regression coefficients corresponding to $W$.
However, we would like to emphasize that {\it not all $W$ need to be observed} in our scenario.  That is, some background features may be hidden (i.e., $\cup_j S_j \neq \{1,\dots, q\}$. See the left of Figure~\ref{fig:story} for an example).
Assumption (iii) rules out the scenario where $X$ influences $Y$ indirectly via the mediator  $Y$.
Assumption (iv) means that the causal relation between $X$ and $Y$ remains constant across environments and is a typical assumption in the literature of discussing the relation between invariance and causality.
Assumption (v) may be vague at this point, but it is crucial for our method and is the key insight why it is possible to obtain some causal information from observational data. Generally speaking, assumption (v) allows us to leverage multiple environments like intervention 
\cite{eberhardt2007interventions, peters2017} to extract causal information. Technically speaking, assumption (v) provides a sufficient condition that the causal structure becomes identifiable. We will provide a mathematical definition for assumption (v) later in Section~\ref{sec:justification}.


\section{Identifiying causal drivers in the presence of unobserved confounders} \label{sec:justification}

In this section, we discuss our proposed method that shows how the variability of regression coefficients can be used to identify causal relations in the presence of unobserved confounders. Before we get to the main result of this section, Theorem \ref{thm:main-res}, we first explain the difficulty of identifying causal drives if we only consider the corresponding regression coefficients. Further, we explain why the stability of regression coefficients can only be a necessary and sufficient property to detect causal influence if certain assumptions about the data generating process are made. All proofs are deferred to Appendix~\ref{apd:proof-normal} and \ref{apd:rotation-invariant-distribution}.


{\bf The problem of biased regression coefficients:} Recall the scenario in Figure~\ref{fig:story} under which we have three datasets $\{(Y, X, W_{S_i})\}_{i=1}^3$. For simplicity, we assume that we have infinite samples for each dataset and obtain the regression coefficients $\hat{\beta}(S_j)$ of $Y$ on $X$ given the background features $W_{S_j}$ for $j=1,2,3$. A straightforward approach to test whether X causally influences Y would be to test whether $\hat{\beta}(S_j)$ is zero or not. However, due to latent confounders $\{Z_1, Z_2, W_4, W_9\}$, $\hat{\beta}(S_j)$ is biased even if we have infinite samples unless $W_{S_j}$ blocks the confounding path entirely. To see this, for any subset $S$, $\hat{\beta}(S)$ is given by
\begin{equation} \label{eq:reg-coe-1}
	\begin{split}
		\hat{\beta}(S)
		&= (\Ib_d, \zero) \Var \left(X, W_S\right)^{-1}
		\Cov \left( (X, W_S), Y \right) \\
		&= \beta + \Cb(S) \gamma
	\end{split}
\end{equation}
where $\Db_{S} = \diag (\sigma_{w_S})$ and
\begin{equation} \label{eq:Cb}
\begin{split}
	&(\Cb(S)^\top)_S \\
	&= (\Db_x^{-1} \Bb (\Ib+\Bb^\top \Db_x^{-1} \Bb + \Ab_{S}^\top \Db_{S} \Ab_{S})^{-1} \Ab_{S^c}^\top)^\top,  \\
	&(\Cb(S)^\top)_{S^c} = \zero.
\end{split}
\end{equation}

Hence, biased regression coefficients seemingly render the problem of identifying causal drivers by testing $ \hat{\beta} \neq 0$ impossible due to hidden confounders.
However, a key observation is that $\hat{\beta}(S_j)$ is 'relatively' close to $\beta$  (on the scale of  $\|\beta\|$) for all $j=1,2,3$ as long as $\|\beta\|$ is large enough. That is, all vectors $\hat{\beta}(S_1), \hat{\beta}(S_2), \hat{\beta}(S_3)$ are similar when $X$ is the cause of $Y$, which may reveal the causal relation between $X$ and $Y$ by virtue of \eqref{eq:reg-coe-1}. As a result, we might suspect that if $\hat{\beta}(S_j)$ are close to each other for all $j$ and are not zero, then X indeed causally influences $Y$.

{\bf The problem of asymmetric hidden confounding:}
Although we have argued that causal influence implies stable regression coefficients, we do not have the reverse implication. That is, stability of orientation will indicate that $X$ causally influences $Y$. In fact, all vectors $ \hat{\beta}(S_1), \hat{\beta}(S_2), \hat{\beta}(S_3) $ can be similar even when $\|\beta\|$ is small, which leads to stable regression coefficients even if the relation between $X$ and $Y$ is purely due to hidden confounding. To see this, consider the scenario in which all elements of $A$ and $\gamma$ are one. In that case, the confounders $\Cb(S_j) \gamma$ are similar for all $S_j$, and so  $\hat{\beta}(S_1)$, $\hat{\beta}(S_2)$, $\hat{\beta}(S_3)$ are close to each other even if $\|\beta\|=0$.
This argument shows that we need additional assumptions to ensure that stable regression coefficients imply true causal relation.

The example above shows that, in general, stable regression coefficients are not enough to ensure true causal relation. Hence, we need to restrict the confounding behavior in the sense of Assumption~\ref{amp:causal}.(v). Formally, we use the notion of spherical symmetry to obtain non-biased hidden confounding.

\begin{definition}\citep{bryc2012normal} \label{def:sherically-symmetric-1}	
	A random vector $\nu=(\nu_1, \dots, \nu_q)^\top$ is spherically symmetric if $e^\top \nu \sim \nu_1,$ for all $e$ provided that $\|e\|=1$.	
\end{definition}

Examples of spherically symmetric distributions include multivariate normal distributions with zero mean and identity covariance matrix and uniform distributions on the unit sphere. The former is a typical prior distribution in Bayesian statistics, and the latter provides an example where the components $\{\nu_i\}_{i=1}^q$ are not necessarily independent. Spherical symmetry is a popular notion in multivariate analysis~\cite{fang2018symmetric} and has a wide range of applications in Bayesian statistics~\cite{maruyama2008admissibility, fourdrinier2010robust, maruyama2014robust}, shape analysis in object recognition~\cite{hamsici2009rotation} and causal inference~\cite{janzing2018detecting, janzing2019causal}.

The symmetry assumption ensures that $\Cb(S_j) \gamma$ in \eqref{eq:reg-coe-1} varies with respect to different $S_j$. As a result, if $\|\beta\| = 0$, all vectors $\lbrace \hat{\beta}(S_1), \hat{\beta}(S_2), \hat{\beta}(S_3) \rbrace$ are distinct.
However, there is still the problem that the bias terms of different regression coefficients $\hat{\beta}(S_j)$ are correlated. As we will see later, this problem can be solved as long as $q$ is large enough.

We admit that one can easily construct real-world examples where spherical symmetry is violated (for instance, if many features have a positive influence on the target). However, the qualitative result regarding the relation between regression stability weak confoundedness persists for many other priors. Further, it is common in causal inference that {\it provable} results require strong assumptions.


{\bf The equivalence between causation and stable regression coefficients: }
After we discussed our assumption on the prior distribution, we now formalize what it means that regression coefficients remain stable. In particular, we develop a stability statistic that relies on the variability of the regression coefficients and relates its asymptotic properties to causal links between $X$ and $Y$.

\begin{definition}[{\bf Stability statistic}]
	Given some vectors $\{v_j\}_{j=1}^m$ in $\RR^d$, the stability statistic is defined as
	\begin{equation}\label{eq:var}
		\Vcal := \Vcal(v_1, \dots, v_m) = 1 - \frac{\|m^{-1} \sum_{j=1}^{m} v_j \|^2}{ m^{-1} \sum_{j=1}^{m}\| v_j \|^2}.
	\end{equation}
\end{definition}

Note that if $v_1 = \dots = v_m$, then $\Vcal =0$. Further, if $v_1, \dots, v_m$ are uniformly distributed on the unit sphere, we get $\Vcal \rightarrow 1$ as $m\rightarrow \infty$. The closer $v_1, \dots v_n$ are, the smaller $\Vcal$ gets. Hence, $\Vcal$ quantifies the variability of $v_1, \dots , v_n$.

We are ready to present our main result addressing the equivalence between causation and stable regression coefficients. In particular, our main theorem enables us to solve our crucial tasks: (i) infer if $\|\beta\| \neq 0$ and (ii) estimate $\beta$. Note that all results are established in the infinite sample limit, and the randomness here only comes from the prior. 

\begin{theorem}\label{thm:main-res}
	Let $S_1, \dots , S_m \subsetneq \{1,\dots , q\}$ and $\|\beta\|=\rho_\beta$. Assume that $\gamma$ is a spherically symmetric random vector such that $\EE\|\gamma\|^4 = O(\rho_\gamma^4/q^2)$ and that
	\begin{equation} \label{eq:confoundingSte-3}
		\frac{1}{ m} \sum_{j=1}^{m} \tr [\Cb(S_j)^\top \Cb(S_j)]   = o(q), \ \ \ \text{ as $q\rightarrow\infty$, }
	\end{equation}
	where $\Cb(S)$ is defined in \eqref{eq:Cb}. Then, it holds that (i) $\frac{1}{m}\sum_{j=1}^{m} \hat{\beta}(S_j) \stackrel{p}{\rightarrow} \beta$ and (ii)
	\begin{equation*}
		\Vcal(\hat{\beta}(S_1), \dots, \hat{\beta}(S_m))  \stackrel{p}{\rightarrow} 
		\begin{cases}
			0 &\text{ if $\rho_\beta > 0$ } \\
			c  & \text{ if $\rho_\beta = 0$ }
		\end{cases},
	\end{equation*}
	where $c>0$ is a deterministic constant depending on $\{\Cb(S_j)\}_{j=1}^m$.
\end{theorem}

(i) in Theorem~\ref{thm:main-res} shows that the confounding effect can be corrected by averaging when the strength of confounding increases slowly enough in $q$, and $\frac{1}{m}\sum_{j=1}^{m} \hat{\beta}(S_j)$ is a consistent estimator of $\beta$. (ii) in Theorem~\ref{thm:main-res} justifies that $\Vcal$ indeed is a reasonable test statistic for:
\begin{equation}\label{eq:hypothesis}
	H_0:\|\beta\|=0,  \ \ \ H_A: \|\beta\| \neq 0.
\end{equation}
We will discuss the details of hypothesis testing for causality in the next section. 

Condition \eqref{eq:confoundingSte-3} relates to mild correlation between $\{\hat{\beta}(S_j)\}_{j=1}^m$ caused by confounders between $X$ and $W$. Theorem~\ref{thm:main-res} shows that the confounding effect can be corrected by averaging when the strength of confounding increases slowly enough in $q$. Note that we do not assume that $\{W_{S_j}\}_{j=1}^m$ covers all background variables $W$, i.e. we can have $\bigcup_{j=1}^m S_j \neq \{1,...q\}$. Thus, Theorem \ref{thm:main-res} addresses also the case of hidden confounding. In Appendix~\ref{apd:example} we give  two examples where the conditions of Theorem~\ref{thm:main-res} are satisfied. 

We would like to emphasize that our idea of detecting causes and its analysis does not highly depend on the generative process \eqref{eq:se1}, and our simplified models are needed only to allow us to study and present our idea for the relationship between causality and regression coefficients' stability in a simple manner. The goal of this paper is to understand why regression coefficients' stability is a reasonable criterion for detecting causal relations instead of deriving theorems in the most general form.

\section{Statistical testing for causal influence}\label{sec:permuation-test}

This section discusses how to conduct statistical tests for hypothesis
\eqref{eq:hypothesis}. As discussed, Theorem  \ref{thm:main-res} shows that $\Vcal$ converges a positive number under the null hypothesis $H_0$ and converges to zero under the alternative hypothesis when the sample size is infinite and $q$ goes to infinity. Given a realization of $\gamma$, we do not have any finite sample statements about $\Vcal$, which depend on the causal structure parameters and the noise $N_y$. In particular, when the noise $N_y$ is normally distributed, the numerator and denominator of $\Vcal$ follows two correlated noncentral chi-square distributions whose covariance depends on $\{S_j\}$ and the model parameters $\Ab$ and $\Bb$. In other words, $\Vcal$ is not a pivotal statistic, i.e., $\Vcal$ is a statistic whose null distribution depends on some unknown parameters. This observation means that some unknown model parameters are needed to compute the null distribution of $\Vcal$ as well as p-values. However, it may be intractable to infer the parameters of the model due to high dimensionality. Even worse, the model may be misspecified. To address those problems, we recommend permutation tests~\cite{fisher1936coefficient, anderson1999empirical, winkler2014permutation,hemerik2020permutation} which is a simple yet effective model-free method to conduct statistical inference and provides exact type I error control.

The permutation test is robust against violations of some assumptions such as normality or homoscedasticity. Appendix~\ref{apd:permutation-test} gives a brief introduction as well as the definition of exchangeability.
To present permutation tests, given data $\Xb=(x_1, \dots, x_\ell)$ and a permutation $\pi$, we let $\Xb^\pi :=(x_{\pi(1)}, \dots, x_{\pi(\ell)})^\top$
be the permutation of $\Xb$ according to $\pi$. We summarize our inference procedure in Algorithm~\ref{alg:PT}, referred to as the permutation test. We use model averaging in Algorithm~\ref{alg:model-averaging} in Appendix~\ref{apd:model-averaging} to estimate $\hat \gamma$ since it can be incorporated with our setting without huge additional computations. See Appendix~\ref{apd:model-averaging} for a brief introduction. Note that \eqref{eq:Yhat} in Algorithm~\ref{alg:PT} is computable since $\hat \gamma_S = 0$ if $W_{S}$ is unobserved.  

\begin{algorithm}[ht]
	\caption{Permutation Test} \label{alg:PT}
	\begin{algorithmic}
		\STATE {\bfseries input:} $m$ datasets $\{(Y, X, W_{S_j})\}_{j=1}^m$, a estimation of $\gamma$ $\hat \gamma$, the number of permutations $M$.
		\STATE Compute $\Vcal_0 = \Vcal(\hat{\beta}(S_1), \dots, \hat{\beta}(S_m)) $ via \eqref{eq:var} where $\hat{\beta}(S)$ are the regression coefficients of $X$ from regressing $Y$ on  $X, W_S$.
		\FOR{$i=1$ {\bfseries to} $M$}
		\STATE{Generate permutations $\pi_i$ uniformly.}
		\STATE Compute
		\begin{equation} \label{eq:Yhat}
			\hat{\Yb}(\pi_i) = \Wb \hat{\gamma}+\hat{\Nb}^{\pi_i}
		\end{equation}
		where $\hat{\Nb} = \Yb - \Wb \hat{\gamma} = (\hat n_{y,1}, \dots, \hat n_{y,\ell})^\top$ and
		\STATE $\hat{\Nb}^{\pi_i} = (\hat n_{y,\pi(1)}, \dots, \hat n_{y,\pi(\ell)})^\top.$
		\STATE Compute $\Vcal_i =\Vcal(\hat{\beta}^{(i)}(S_1), \dots, \hat{\beta}^{(i)}(S_m)) $ via \eqref{eq:var} where  $\hat{\beta}^{(i)}(S)$ is regression coefficients of $X$ from regressing $\hat{\Yb}(\pi_i)$ on  $X, W_S$.
		\ENDFOR
		\STATE {\bfseries output:} {p-value $p = (\sum_{i=1}^{M} \ind \lbrace \Vcal_i \geq \Vcal_0 \rbrace + 1)/ (M+1).$}
	\end{algorithmic}
\end{algorithm}

The following theorem shows how the type I error can be controlled, which thus justifies Algorithm~\ref{alg:PT}. In particular, it proves that $((\Yb, \Xb, \Wb)$, $(\hat{\Yb}(\pi_1), \Xb, \Wb)$, ..., $(\hat{\Yb}(\pi_M), \Xb, \Wb))$ is approximately {\it exchangeable}, where $\hat{\Yb}(\pi_i)$ is defined in \eqref{eq:Yhat}.
A similar argument for permutation tests also appears in \cite{berrett2019conditional, barber2020testing}.

\begin{theorem}\label{cor:gaussian-case}
	Suppose that $N_y \sim \Ncal(0, \sigma_y^2)$ and $M$ permutations $\{\pi_i\}_{i=1}^M$ are generated uniformly. Given an estimation $\hat{\gamma}$ of $\gamma$ and the p-value $p$ obtained by Algorithm \ref{alg:PT}, under the null hypothesis $H_0$ in \eqref{eq:hypothesis}, the type I error is bounded from above by 
	\begin{equation*}
		\PP(p \leq \alpha \mid \Xb, \Wb , \gamma, \hat \gamma) \leq \alpha + \frac{ \sqrt{M} \|\Wb (\gamma -\hat{\gamma})\|}{2 \sigma_y}.
	\end{equation*}
\end{theorem}

Hence,  the type I error is small whenever the prediction error is small.  Lasso, principal components regression, variable subset selection, and partial least squares all can be used to get the residuals in the high dimensional setting. In particular, ridge regression and model averaging are more desirable due to their superior ability to minimize the prediction error.

\section{Numerical studies} \label{sec:experiments}

In this section, synthetic and real datasets are used to gain additional insight into the performance of our method. Further, we compare our proposed method against existing tests for hypothesis \eqref{eq:hypothesis}. The code for all experiments is given in the supplementary material.

\begin{figure*}[ht]
	\centering
	\includegraphics[width=0.32\textwidth]{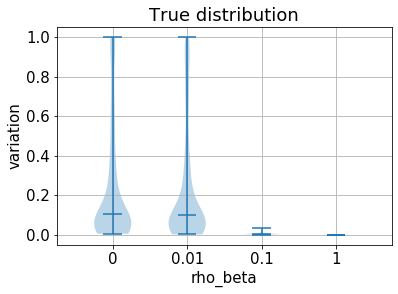}
	\includegraphics[width=0.32\textwidth]{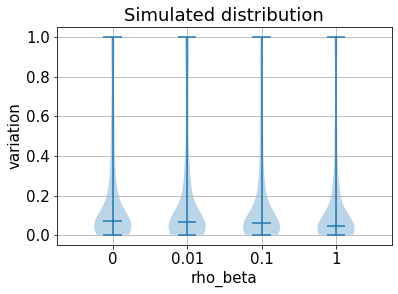}
	\includegraphics[width=0.32\textwidth]{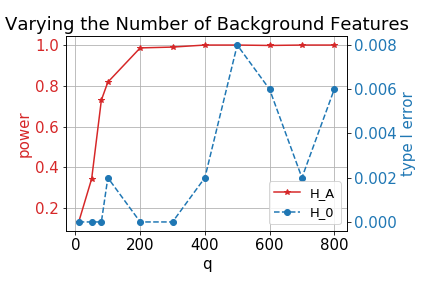}
	\caption{Left: the true distribution of the variation $\Vcal$. Middle: the null distribution generated by the permutation test. Right: The Power and Type I error with varying the number of background features.} \label{fig:sem1}
\end{figure*}

\begin{table*}[ht]
	\centering
	\caption{ $q=1000, r=5, \ell =100, k=10, m = 1000$ and $\rho_\gamma=10, \rho_\beta = 0, 1.5$. } \label{table:sem1}
	\resizebox{0.8\textwidth}{!}{
		\begin{tabular}{cccccccccccc}
			\toprule
			&          & \multicolumn{5}{c}{Setting 1}        & \multicolumn{5}{c}{Setting 2}        \\
			\cmidrule(lr){3-7} \cmidrule(lr){8-12}
			& $\alpha$ & \textbf{RS}    & JS    & BM    & FL    & DR    & \textbf{RS}    & JS    & BM    & FL    & DR    
			\\ 
			\cmidrule(lr){1-1} \cmidrule(lr){2-2} \cmidrule(lr){3-7} \cmidrule(lr){8-12}
			type I           & 0.05     & \textbf{0.053} & 0.172 & 0.026 & 0.001 & 0.001 & \textbf{0.040} & 0.177 & 0.033 & 0.017 & 0.024 \\
			error            & 0.01     & \textbf{0.010} & 0.140 & 0.007 & 0.000 & 0.000 & \textbf{0.009} & 0.146 & 0.008 & 0.005 & 0.005 \\ 
			\cmidrule(lr){1-1} \cmidrule(lr){2-2} \cmidrule(lr){3-7} \cmidrule(lr){8-12}
			power    & 0.05     & \textbf{0.817} & 0.126 & 0.799 & 0.029 & 0.025 & \textbf{0.801} & 0.138 & 0.728 & 0.519 & 0.567 \\
			& 0.01     & \textbf{0.623 }& 0.159 & 0.628 & 0.02  & 0.017 & \textbf{0.589} & 0.163 & 0.522 & 0.342 & 0.345 \\ 
			\bottomrule
	\end{tabular}}
\end{table*}

\subsection{Simulations}\label{sec:simulations}

{\bf Data generating procedure:} The data $(\Yb, \Xb, \Wb)$ is generated via SEM~\eqref{eq:se1} with the following parameters:
(1) Draw the entries of $\Ab, \Bb$ from independent Gaussian distributions $\Ncal(0,1/d)$ and $\Ncal(0,1/q)$, respectively.
(2) Draw the parameters $\beta, \gamma$ from the uniform distribution on the sphere with radius $\rho_\beta, \rho_\gamma$, respectively.
(3) Draw $\ell$ samples of each $N_{\square}$ independently from the standard multivariate Gaussian distribution $\Ncal(0, \Ib)$ for $\square=z,w,x,y$.

{\bf Random selection procedure  (RS): }
If we only have a single dataset with a large number of background features $ \Wb$ rather than multiple datasets, we artificially generate multiple datasets $\{(\Yb, \Xb, \Wb_{S_j})\}_{j=1}^m$ by uniformly sampling $m$ subsets $S_j \subset \{1,\dots, q\}$ with size $|S_j| = k$. We call this method  random selection procedure.

{\bf Algorithm~\ref{alg:PT}+\ref{alg:model-averaging}+RS:} To apply the permutation test in the numerical studies, we 
(i) use the {\it random selection procedure} to create $\{(\Yb, \Xb, \Wb_{S_j})\}_{j=1}^m$,
(ii) execute Algorithm~\ref{alg:model-averaging} in Appendix~\ref{apd:model-averaging} to estimate $\gamma$, and (iii) run Algorithm~\ref{alg:PT} to test the hypothesis in \eqref{eq:hypothesis}. For the ease of notation, we often refer to RS when we mean the procedure Algorithm~\ref{alg:PT}+\ref{alg:model-averaging}+RS.

{\bf Basis results: } We generate data $(\Yb, \Xb, \Wb)$ by using the {\it data generating procedure} with the following parameters: $d=1$, $q=1000$, $r=5$, $\ell=100$, $\rho_\gamma =10$ and $\rho_\beta = 0, 0.01, 0.1, 1.0$. For each $\rho_\beta$, we repeated the experiment 1000 times. Note that, $\beta$ and $\gamma$ are only drawn once and were then fixed for all 1000 repetitions. The only randomness comes from generating new noises $N_y$ for each experiment. The result of runing {\it Algorithm~\ref{alg:PT}+\ref{alg:model-averaging}+RS} with $m = 100, M=1000, k=10$
is shown in Figure~\ref{fig:sem1} (left and middle). We can see in the left figure that the distribution of $\Vcal$ concentrates at zero for $\rho_\beta=1$, which aligns with our findings in Theorem~\ref{thm:main-2}.
The middle figure in Figure~\ref{fig:sem1} shows that the null distribution generated by the permutation test is similar to the true null distribution for different $\rho_\beta$, which verifies Theorem~\ref{cor:gaussian-case} empirically.

{\bf Results of varying parameters: } Again, we use the {\it data generating procedure} but with the parameters $d=3$, $r=5$, $\ell=300$, $\rho_\gamma =2.5$, and {\it Algorithm~\ref{alg:PT}+\ref{alg:model-averaging}+RS} with $m =40, M=400, k=3$ to investigate the impact of different choices of $q$. The result is shown in Figure~\ref{fig:sem1} right, thereby we repeat the experiments 500 times and take the average for each point in the figure.
We can see that increasing the number of background features can increase the power that approaches one, which aligns with Theorem~\ref{thm:main-res}. We also conduct experiments with varying $\rho_\beta$, $d$, and $k$ in Appendix~\ref{apd:additional-experiments}.

{\bf Comparisons: }
Next, we compare our proposed approach (Algorithm~\ref{alg:PT}+\ref{alg:model-averaging}+RS) to existing methods. Thereby, we use the {\it data generating procedure} to produce 1000 data for the null and alternative hypothesis with $q=1000$ (see caption in Table~\ref{table:sem1} for all hyperparameters), where we consider two settings:

\noindent
{\it Setting 1:} All $W_1, \dots, W_{1000}$ are accessible.

\noindent
{\it Setting 2:} $W_1$, $\dots$, $W_{700}$ are latent, and only $W_{701}$, $\dots$, $W_{1000}$ are accessible.

In this comparison, we run {\it Algorithm~\ref{alg:PT}+\ref{alg:model-averaging}+RS} with $m = 1000$, $k=10$ and $M =1000$, referred to as "RS".
The following methods are included in our experiment:

(i)"JS" is a test for non-confounding proposed in \cite{janzing2018detecting} and particularly designed for the case that all elements of $W$ are hidden. Under an ICA-based model, the authors derive a test statistic for the null hypothesis $\rho_\beta/\rho_\gamma=0$ and a simple approximation of the null distribution of their proposed statistic.

\noindent
(ii) "BM" is a high-dimensional test based on ridge projections, proposed in \cite{buhlmann2013statistical}. This test is based on a bias-corrected estimation and an asymptotic upper bound of its distribution.

\noindent
(iii) "FL" is the Freedman-Lane HD test proposed by \cite{hemerik2020permutation} with the test statistic based on the generalized partial correlation.  

\noindent
(iv) "DR" is the Double Residualization method in \cite{hemerik2020permutation}, which residualizes both $Y$ and $X$ and tests the sample correlation.

Note that FL and DR use ridge regression for  $\gamma$, while our method uses model averaging (Algorithm~\ref{alg:model-averaging}).  BM, FL, and DR are designed for the conditional independence tests and do not consider the presence of unobserved confounders. They only use the assumption of linear models and do not restrict any generating processes of $\gamma$. In contrast, RS and JS aim to correct or detect hidden confounding and rely on strong assumptions.

The summary is shown in Table~\ref{table:sem1}. We can see that RS shows the highest power and its type I error is closest to $\alpha$ in both settings. It may be counter-intuitive that FL and DR work better in Setting 2, but since $\gamma$ is drawn from the prior with zero mean, the confounding effect will be offset when $q$ is large. Thus, the prediction error increases in Setting 1 due to this particular prior. We would like to remind readers that BM, FL, and DR are designed for more general model settings. JS does not use any information about $W$ and does not take into account noise terms $N_y$, which may explain the bad performance of JS. This observation also shows the importance of permutation tests that do not rely on the model's parameters.
Due to the limit of space, we include additional experiments in Appendix~\ref{apd:additional-experiments} where we also examine more settings and model parameters.


\subsection{Real datasets} \label{sec:real-dataset}

This section presents an analysis for a real dataset about the educational attainment of teenagers~\cite{rouse1995democratization}. We work with simplified data as provided in \cite{stock2012introduction}.\footnote{Readers can download data and see the documentation in \url{https://www.princeton.edu/~mwatson/Stock-Watson_3u/Students/Stock-Watson-EmpiricalExercises-DataSets.htm}}
In this data, for 4739 pupils from approximately 1100 US high schools in different states, 13 attributes were recorded in 1980. We split the whole dataset into several sub-datasets based on the attribute of the state hourly wage in manufacturing ("stwmfg80").
In this experiment, we want to examine how different attributes of pupils affect the composite test score denoted by $Y$. The attribute "years of education completed", the total years that students get their final degree, and "tuition", the average intuition in four-year state college for each state, are dropped due to the time order and colinearity, respectively. 
Based on Assumption~\ref{amp:causal}, our method requires that causal influences of $X$ on $Y$ are invariant across different environments. Thus, the attributes we choose to test for potential causal influence are: (1) The boolean information of whether the parents are college graduates, the corresponding features are denoted by "momcoll" and "dadcoll", and (2) The distance from a four-year college to students' school, denoted by "dist".
We only use those sub-datasets whose sample size is more than 70. As a result, 11 attributes, 3908 samples, and 20 sub-dataset remain in this analysis. Other attributes include as background features in our analysis are gender ("Female"), race ("black", "Hispanic"), scores on relevant achievement tests ("bytest"), whether the family income is larger than 25,000 ("income"), whether the family owns a house ("own home"), whether the student's school is in an urban area ("urban") and the unemployment rate of the country where the school is located ("cue80"). 

We perform our Algorithm~\ref{alg:PT}+\ref{alg:model-averaging} with $M=1000$ and run the ordinary least squares (OLS) regression for baseline comparison.
For the complete result of the ordinary linear regression, see Appendix~\ref{apd:real-data}.
The result is summarized in Table~\ref{tab:real-data}. We can see for OLS that all features are statistically significant as those p-values are extremely small. This result is not reasonable since "dist" should not causally influence students' composite test scores. Indeed, it is highly possible that some hidden confounders exist between "dist", and another attribute as "dist" is a proxy for the type of schools. In contrast to OLS, our proposed method produces a large p-value for "dist".
On the other hand, "momcoll" and "dadcoll" can influence composite test scores in many aspects. For example, the parents who are college graduates may emphasize education more. Both methods yield a small p-value for  "momcoll" and "dadcoll", which verifies our reasoning. Note that we test "momcoll" and "dadcoll" together in our proposed method.
Even though we do not know the ground truth in this case,  the results from Algorithm~\ref{alg:PT} are more plausible than OLS regression because of potentially hidden common causes.

\begin{table}[tbh]
	\caption{p-values conducted by OLS and our proposed procedure} \label{tab:real-data}
	\centering
	\resizebox{0.4\textwidth}{!}{
		\begin{tabular}{ccccc}
			\toprule
			Method & momcoll & dadcoll   & dist  \\
			\cmidrule(lr){1-1} \cmidrule(lr){2-3} \cmidrule(lr){4-4}
			Algorithm~\ref{alg:PT}               & \multicolumn{2}{c}{0.0}                                    & 0.970                               \\
			OLS                   &     3.10e-09      & 1.36e-18 & 1.08e-06   \\
			\bottomrule
	\end{tabular}}
\end{table}

\section{Conclusions and discussions}

We proposed a procedure to identify the linear influence of a feature on a target variable in a scenario with multiple features as background conditions. 
We showed that stability of regression coefficients with respect to different selections of background variables 
(formalized by our stability statistic $\Vcal$) indicates that the statistical association relies on 
a direct causal influence, as opposed to hidden confounders.
Although assumption~\ref{amp:causal} might seem very restrictive and needs to be carefully examined in practice, one should be aware of the difficulty of inferring causal relations from observational data in the presence of hidden confounders. Hence, it is natural to postulate assumptions about the underlying generation process to be able to prove correctness of methods which hopefully work also under more general circumstances. Since our model is based on a large number 
of possibly correlated confounders it complements previous studies using stability of regression coefficients as indicators for associations to be causal.

\newpage
\appendix
\onecolumn
\section{Permutation Tests for the conditional independence test}\label{apd:permutation-test}
This section briefly introduces permutation tests~\cite{fisher1936coefficient, anderson1999empirical} with a focus on conditional independence. Permutation tests provide exact type I error control and are robust to misspecified assumptions such as normality and homoscedasticity.

Exchangeability is the crucial concept for permutation tests. Formally, given a finite or infinite sequence of random variables $V = (V_1, V_2, \dots)^\top$, we say that $V$ is {\it exchangeable} if for any finite permutation $\pi$ of the set of indices $\{1, 2, \dots\}$ we have
\begin{equation*}
	V \sim V^\pi = (V_{\pi(1)}, V_{\pi(2)}, \dots).
\end{equation*}
For example, $\Nb_y$ is exchangeable since the assumption of i.i.d. samples implies
\begin{equation*}
	\Nb_y \sim \Nb_y^\pi=(n_{y, {\pi(1)}}, \dots, n_{y, {\pi(\ell)}})^\top.
\end{equation*}

Before discussing the conditional independence test, we first recap how to conduct the independence test in the simplest case as follows.
Recall that SEM \eqref{eq:se1} is used for the data generating process.
Consider $q=0$ and the case $H_0: Y \independent X$. Given $M$ permutations $\{\pi_i\}_{i=1}^M$ of $\lbrace 1,\dots, \ell\rbrace$ generated uniformly, $((\Yb, \Xb), (\Yb, \Xb^{\pi_1}), \dots, (\Yb, \Xb^{\pi_M}))$ is exchangeable by the assumption of i.i.d. samples.
Given any valid statistic $T(\Yb, \Xb)$ such as the correlation between $\Yb$ and $\Xb$, $p$-values are of the form
\begin{equation*}
	p_0 = \frac{\sum_{i=1}^{M} \ind \lbrace T(\Yb, \Xb^{\pi_i}) \geq T(\Yb, \Xb) \rbrace + 1}{M+1}.
\end{equation*}
where $\ind$ is the indicator function.
Then it holds for permutation tests that the exact type I error can be bounded as follows by the exchangeability, i.e.,
\begin{equation*}
	\PP(p_0<\alpha|H_0, \Yb, \Xb) < \alpha
\end{equation*}

However, if $q>0$ and $X, W$ are correlated, $((\Yb, \Xb^{\pi_1}, \Wb), \dots, (\Yb, \Xb^{\pi_M}, \Wb) )$ is not exchangeable due to the correlation between $X$ and $W$. The confounding effect must be considered or removed to fix this issue, and several different permutation tests are proposed. The following sections discuss two types of permutation tests for high-dimensional conditional independence tests. See \citet{winkler2014permutation} for more different types of permutation tests.

\subsection{Freedman-Lane permutation method}

Even if we cannot simply permute indices of $\Xb$, we note that $\Nb_y$ is exchangeable. As a result, if the null hypothesis $H_0: Y \independent X \mid W$ holds, $((\Yb, \Xb, \Wb), (\bar{\Yb}(\pi_1), \Xb, \Wb), \dots, (\bar{\Yb}(\pi_M), \Xb, \Wb))$ is exchangeable where $\bar{\Yb}^{\pi_i} = \Wb \gamma  + \Nb_y^{\pi_i}$ and $\pi_i$ is drawn uniformly. This is called Freedman-Lane permutation method \cite{freedman1983nonstochastic} and requires the estimation of the nuisance parameters $\gamma$. Every high-dimensional regression method can be combined with this permutation method. For instance, \cite{hemerik2020permutation} proposed the ridge regression to estimate $\gamma$:
\begin{equation}
	\begin{split}
		\hat{\gamma}_{\text{ridge}} =& \argmin_\gamma \left( \|\Yb - \Wb \gamma\|^2+ \lambda \|\gamma\|^2 \right).
	\end{split}
\end{equation}
Then $\hat{\Nb}_y = \Yb - \Wb \hat{\gamma}_{\text{ridge}}$ is approximated exchangeable, so we can generate copies of $\Yb$ by $\hat{\Yb}(\pi_i) = \Wb \hat{\gamma}_{\text{ridge}}  + \hat{\Nb}_y^{\pi_i}$ with $M$ permutations $\{\pi_i\}_{i=1}^M$. Given any valid statistic $T(\Yb, \Xb, \Wb)$ such as the partial correlation between $\Yb$ and $\Xb$ given $\Wb$, $p$-values are of the form
\begin{equation*}
	p = \frac{\sum_{i=1}^{M} \ind \lbrace T(\hat{\Yb}(\pi_i), \Xb, \Wb) \geq T(\Yb, \Xb, \Wb) \rbrace + 1}{M+1}.
\end{equation*}
Then we have the approximated type I error control:
\begin{equation*}
	\PP(p<\alpha|H_0, \Yb, \Xb, \Wb) < \alpha + \text{an approximation error},
\end{equation*}
where the approximation error is due to the fact that we used $\hat{\gamma}_{\text{ridge}}$ instead of ${\gamma}$.
See Proposition~\ref{cor:gaussian-case} for an example of the approximation error.

\subsection{The conditional permutation tests} \label{apd:conditional-permutation-tests}
To correct the confounding effect, \citet{berrett2019conditional} proposed the conditional permutation test (CPT), which utilizes the conditional distribution of $X$ given $W$ that is assumed to be known. They argued that in semi-supervised learning, unlabelled data $(X, W)$ are plentiful, and labeled data $(Y, X, W)$ are rare. As a result, to accurately estimate the conditional distribution of $X$ given $W$ is more likely but to test for independence with $Y$ remains challenging due to the limited sample size of the labeled data.
Instead of using independent noises as the Freedman-Lane permutation method, CPT generates permutations $\{\pi_i\}_{i=1}^M$ non-uniformly with the probability
\begin{equation} \label{eq:CPT-permutation-dis}
	\PP \left(\pi_{j} =\pi \mid \Yb, \Xb, \Wb \right)= \frac{q\left(\Xb^{\pi} \mid \Wb\right)}{\sum_{\pi^{\prime} \in \Scal_n} q\left(\Xb^{\pi^{\prime}} \mid \Wb\right)}
\end{equation}
where $q(\Xb|\Wb)$ is the conditional density of $\Xb$ given $\Wb$ and $\Scal_n$ is the permutation group over $\lbrace 1,\dots, \ell \rbrace$. Then Theorem 1 in \cite {berrett2019conditional} showed that $((\Yb, \Xb, \Wb), (\Yb, \Xb^{\pi_1}, \Wb), \dots, (\Yb, \Xb^{\pi_M}, \Wb))$ is exchangeable. For sampling permutations from the distribution \eqref{eq:CPT-permutation-dis}, they implemented Markov chain Monte Carlo sampler .

\section{Model Averaging} \label{apd:model-averaging}

This section introduces model averaging within the frequentist paradigm. Model averaging is an approach of averaging over the estimators for a set of candidate models using predefined or data-adaptive weights.
Formally, under the null hypothesis \eqref{eq:hypothesis}, we are concerned with this model:
\begin{equation}\label{eq:se3}
	y_i \leftarrow  \gamma^\top w_i + n_{y,i}, \text{ for $i=1,\dots, \ell$.}
\end{equation}
Here $n_{y,i}$ is the independent noise such that $\EE n_{y,i}=0$ and $\EE n_{y,i}^2 = \sigma_y^2$ for all $i$ . Given data $(y_i, x_i, w_i)$ consist of $\ell$ i.i.d. samples generated from \eqref{eq:se3}. We write $\Yb=(y_1, \dots, y_\ell)^\top \in \RR^\ell$, $\Xb = (x_1, \dots, x_\ell)^\top \in \RR^{\ell \times d}$, $\Wb=(w_1,\dots, w_\ell)^\top \in \RR^{\ell \times q}$. Given $m$ subsets $\{S_i\}_{i=1}^m$, the estimator $\hat{\gamma}(S_j)$ for $\gamma$ via the $j$-th submodel is
\begin{equation} \label{eq:reg-coef-finite}
	\hat{\gamma}(S_j)_{S_j} =
	\left(
	\Wb_{S_j}^\top \Wb_{S_j}
	\right)^{-1} \Wb_{S_j}^\top \Yb, \ \ \ \hat{\gamma}(S_j)_{S_j^c}=0
\end{equation}
and the averaging estimator of $\gamma$ with respect to a weight vector $w=(w_1,\dots,w_m)^\top$ is
\begin{equation*}\label{eq:ave-estimator}
	\hat{\gamma}(w) = \sum_{i=1}^{m} w_i \hat{\gamma}(S_i),
\end{equation*}
where the weight vector $w$ lies in the unit simplex $\Wcal$ in $\RR^m$, i.e.
\begin{equation*}
	\Wcal =\left\lbrace w \in [0,1]^M: \sum_{i=1}^{m} w_i =1 \right\rbrace .
\end{equation*}
In our paper, we only consider the uniform weighting, i.e. $w_i = w_j$ for all $i,j$ and state our model averaging procedure in Algorithm~\ref{alg:model-averaging}.

\begin{algorithm}[ht]
	\caption{Model Averaging} \label{alg:model-averaging}
	\begin{algorithmic}
		\STATE {\bfseries input:} $m$ datasets $\{(Y, X, W_{S_j})\}_{j=1}^m$, a estimation of $\gamma$
		\FOR{$j=1$ {\bfseries to} $M$}
		\STATE  Compute the estimation $\hat{\gamma} (S_j)$ of coefficients $\gamma$ by regressing $\Yb$ on $\Wb_{S_j}$ via \eqref{eq:reg-coef-finite}
		\ENDFOR
		\STATE {\bfseries output:} {$\hat \gamma = \frac{1}{m} \sum_{j=1}^{m} \hat{\gamma}(S_j)$}
	\end{algorithmic}
\end{algorithm}

However, it is possible to implement other data-adaptive weights to improve the accuracy of the estimator of $\hat{\gamma}(w) $. As we stated in Theorem~\ref{cor:gaussian-case}, an accurate prediction can decrease the type I error. We will discuss some popular model averaging strategies. See \cite{zhang2015consistency} for some consistency results for those strategies.

\subsection{Smoothed AIC and smoothed BIC}
In the $j$-th candidate model $\Wb_{S_j}$ with the size $k_j =|S_j|$, the AIC score is
\begin{equation*}
	\operatorname{AIC}_j= \ell \log \hat{\sigma}_j^2 + 2 k_j
\end{equation*}
and the BIC score is
\begin{equation*}
	\operatorname{BIC}_j = \ell \log \hat{\sigma}_j^2 + 2 k_j  \log (\ell),
\end{equation*}
where $\hat{\sigma}_j^2=\ell^{-1} \|\Yb- W \hat{\gamma}(S_j)\|^2$. We define the weights from AIC and BIC \cite{hjort2003frequentist}:
\begin{equation}\label{eq:w-AICBIC}
	\hat{w}_{\operatorname{xIC}_m, i} = \frac{\exp(-\operatorname{xIC}_i/2)}{\sum_{j=1}^{m} \exp(-\operatorname{xIC}_j/2)}, \ \ \ i=1,\dots, m,
\end{equation}
where $\operatorname{xIC}_i$ is the AIC or BIC score under the $i$-th candidate model. The averaging estimators combined by weights in \eqref{eq:w-AICBIC} are commonly called smoothed AIC (S-AIC) or smoothed BIC (S-BIC) estimators

\subsection{MMA}

Let $\Hb(\wb) = \sum_{j=1}^{m} w_j \Wb_{S_j} (\Wb_{S_j}^\top \Wb_{S_j})^{-1} \Wb_{S_j}^\top$ and $k=(|S_1|,\dots,|S_m|)^\top$. \citet{hansen2007least} proposed choosing weights by minimzing the Mallow criterion  
\begin{equation*}
	\Ccal(w) = \| [\Ib_n - \Hb(w)] \Yb \|^2+2 \hat{\sigma} w^\top k
\end{equation*}
where
\begin{equation*}
	\hat{\sigma}=(\ell-|S_{j^\star}|)^{-1}\|\Yb-W\hat{\gamma}(S_{j^\star})\|, \text{ and } j^\star=\arg \max_j |S_j|.
\end{equation*}
Define the empirical Mallow weight as $\hat{w}_{\operatorname{MMA}}=\arg \min_{w} \Ccal(w)$ which has several good theoretical properties. For example, it can be shown that the mallow criterion is the unbiased of squared error plus a constant  (Lemma 3, \cite{hansen2007least} ), i.e.
\begin{equation}
	\EE \Ccal(w) = \left( W \hat{\gamma}(w) - W \gamma \right)^\top \left( W \hat{\gamma}(w) - W \gamma \right) + \ell \sigma_y^2.
\end{equation}
Thus, the empirical Mallow weight $\hat{w}_{\operatorname{MMA}}$ asymptotically minimizes the squared error (Theorem 1, \citet{hansen2007least}). Moreover, the weight vector obtained by Mallow's criterion has a sparsity property in the sense that a subset of models receives exactly zero weights \cite{feng2020sparsity}.

\section{Normally distributed Prior}\label{apd:proof-normal}

This section our main results for normally distributed prior. For the sake of simplicity, we first impose two conditions, then show general case in the next section. To begin with, let $r=1$, and $\Ab$ and $\Bb$ reduce to $a:=\Ab$ and $b :=\Bb$, respectively. Then, assume $\gamma$ is normally distributed, i.e. 
\begin{equation} \label{eq:Bayesian}
	\gamma \sim \Ncal\left( 0, (\rho_\gamma^2/q) \Ib_q \right).
\end{equation}


Now we are ready to present our first main result addressing the equivalence between causation and stable regression coefficients. In particular, we prove that the cause $X$ can be identified as long as $\|\beta\|$ or $q$ is large enough.  Note that all results are established in the infinite sample limit. That is, the randomness here only comes from the prior \eqref{eq:Bayesian}.

\begin{theorem}\label{thm:main}
	Let  $S_1, \dots , S_m \subsetneq \{1,\dots , q\}$. Define $\Sigmab=[\Sigma_{i j}]_{i,j=1}^m \in \mathbb{R}^{m \times m}$ by
	\begin{equation} \label{eq:sigmab}
		\Sigma_{i j} = \frac{ \|a_{S_i^c \bigcap S_j^c}\|^2}{(1+\|b\|_x^2 +\|a_{S_i}\|_{{S_i}}^2)(1+\|b\|_x^2 +\|a_{S_j}\|_{{S_j}}^2)}.
	\end{equation}
	where $\|b\|_x^2  = b^\top \Db_x b$, $\|a_S\|_{S}^2 = a_S^\top \Db_{S} a_S$. Assume that \eqref{eq:Bayesian} holds, and $\Sigmab$ has full rank with eigendecomposition $\Pb \diag(\lambda_1, \dots, \lambda_m) \Pb^\top$. 
	Let $e^m=(1/\sqrt{m},\dots, 1/\sqrt{m})^\top$.
	
	\noindent
	(a) We have two cases: if $\rho_\beta = 0$, it holds that
	\begin{equation}
		\Vcal\left(\hat{\beta}(S_1), \dots, \hat{\beta}(S_m)\right) \sim 1-\frac{\left|\sum_{i=1}^m c_i \lambda_i^{1/2} g_i \right|^2}{\sum_{i=1}^m \lambda_i g_i^2};
	\end{equation}
	if $\rho_\gamma^2/\rho_\beta^2 \rightarrow 0$, we have $  \EE\Vcal(\hat{\beta}(S_1), \dots, \hat{\beta}(S_m)) = o(1)$. $(c_1,\dots, c_m) = (e^m)^\top \Pb $ and $(g_1, \dots, g_m)^\top \sim \Ncal(0, \Ib_m)$.
	
	\noindent
	(b) Further, assume that 
	\begin{equation} \label{eq:confoundingSte}
		\tr(\Sigmab)/m  = o(q), \ \ \ \text{ as $q\rightarrow\infty$, }
	\end{equation}
	then,
	\begin{equation*}
		\Vcal(\hat{\beta}(S_1), \dots, \hat{\beta}(S_m)) \ \stackrel{p}{\rightarrow}
		\begin{cases}
			0 &\text{ if $\rho_\beta > 0$ } \\
			1 - \frac{(e^m)^\top \Sigmab e^m}{\tr(\Sigmab)}  & \text{ if $\rho_\beta = 0$ },
		\end{cases},
	\end{equation*}
	and $\frac{1}{m}\sum_{j=1}^{m} \hat{\beta}(S_j) \stackrel{p}{\rightarrow} \beta$.
\end{theorem}

The first result states that the upper bound of the mean of $\Vcal$ is decreasing as $\rho_\beta/\rho_\alpha$ increases, while if $\|\beta\|=0$ we get some positive number drawn from some known distributions. The second result says as $q$ increasing, $\Vcal$ converges to zero if and only if $\rho_\beta>0$. As a result, since the distribution of $\Vcal$ under the null and alternative hypothesis are well separated from each other given large enough $\rho_\beta/\rho_\alpha$ or $q$, Theorem~\ref{thm:main} prove the idea shows that $\Vcal$ is a reasonable test statistic for testing whether $\|\beta\|$ is zero. $\Sigmab$ in \eqref{eq:sigmab} relates to the covariance matrix of $\hat{\beta}(S_j)$, so the condition \eqref{eq:confoundingSte} implies mild correlation among $\{\hat{\beta}(S_j)\}_{j=1}^m$ caused by confounders between $X$ and $W$. We provide two examples that show the sufficient condition for \eqref{eq:confoundingSte} is $\|a\|^2 = o(q) $ provided \eqref{eq:subset} holds in Appendix~\ref{apd:example}. 

To show our main results, we first provide some useful calculations. First, we define
\begin{equation} \label{eq:v-distribution}
	v =  (v_1, \dots, v_m) :=\left(\frac{a_{S_1^c}^\top \gamma_{S_1^c}}{1+\|b\|_x^2 + \|a_{S_1}\|_{{S_1}}^2},\dots, \frac{a_{S_m^c}^\top \gamma_{S_m^c}}{1+\|b\|_x^2 + \|a_{S_m}\|_{{S_m}}^2}\right).
\end{equation}
and $e^m:=(1/\sqrt{m},\dots, 1/\sqrt{m})^\top$ is a $m$-dimensional unit vector.
Then, $\gamma \sim \Ncal\left( 0, \frac{\rho_\gamma^2}{q} \Ib_q \right)$ yields $v \sim \Ncal(0, \frac{\rho_\gamma^2}{q} \Sigmab)$ where
\begin{equation*}
	\Sigma_{i j} = \frac{ \|a_{S_i^c \bigcap S_j^c}\|^2}{(1+\|b\|_x^2 +\|a_{S_i}\|_{{S_i}}^2)(1+\|b\|_x^2 + \|a_{S_j}\|_{S_j}^2)}.
\end{equation*}
Recall the regression coefficients $\hat{\beta}(S)$ of $Y$ on $X, W_S$ is
\begin{equation*}
	\hat{\beta}(S) = \beta + \frac{a_{S^c}^\top \gamma_{S^c} }{1 + \|b\|_x^2 + \|a_S\|_{S}^2 } \Db_x^{-1} b,
\end{equation*}
implying
\begin{equation} \label{eq:beta}
	\begin{split}
		\left\| \frac{1}{ m} \sum_{j=1}^m \hat{\beta}(S_j) \right\|^2 =  \left\| \beta + \frac{1}{ \sqrt{m}} (e^m)^\top v \Db_x^{-1} b \right\|^2 =\|\beta\|^2+ \frac{2}{\sqrt{m}} (e^m)^\top v \beta^\top\Db_x^{-1} b + \frac{1}{m}  |(e^m)^\top v |^2 \|\Db_x^{-1} b\|^2, \\
		\frac{1}{m} \sum_{j=1}^{m}\|\hat{\beta}(S_j)\|^2 = \frac{1}{m} \sum_{j=1}^{m}\|  \beta + v_j \Db_x^{-1} b\|^2 = \|\beta\|^2 +  \frac{2}{\sqrt{m}} (e^m)^\top v \beta^\top\Db_x^{-1} b + \frac{1}{ m} \|v\|^2 \|\Db_x^{-1} b\|^2.
	\end{split}
\end{equation}

Using a well-known result of normal distributions \cite{rencher2008linear} that
\begin{equation*} 
	\EE [g^\top \Hb g] = \tr [\Hb \Gb],  \ \ \
	\Var  [g^\top \Hb g]  =  2 \tr [\Hb \Gb \Hb \Gb ]
\end{equation*}
where $\Hb$ is a symmetric matrix and $g \sim \Ncal(0, \Gb)$, we know
\begin{equation} \label{eq:emv-mean}
	\EE \left[ \frac{1}{m}  |(e^m)^\top v |^2 \right]  = \frac{\rho_\gamma^2}{q} \frac{1}{m}  (e^m)^\top \Sigmab e^m, \ \ \
	\EE \left[ \frac{1}{ m} \|v\|^2 \right]= \frac{\rho_\gamma^2}{q} \frac{1}{m} \tr(\Sigmab) ,
\end{equation}
and
\begin{equation} \label{eq:emv-var}
	\begin{split}
		\Var \left[ \frac{1}{m}  |(e^m)^\top v |^2 \right] =& \frac{\rho_\gamma^4}{q^2} \frac{2}{m^2}  ((e^m)^\top \Sigmab e^m)^2 \leq  \frac{\rho_\gamma^4}{q^2} \frac{2}{m^2} \tr(\Sigmab^2), \\
		\Var \left[ \frac{1}{ m} \|v\|^2 \right]  =& \frac{\rho_\gamma^4}{q^2} \frac{2}{m^2}  \tr(\Sigmab^2).
	\end{split}
\end{equation}
In what follows, we will prove Theorem~ \ref{thm:main}.

\subsection{Proof of (a) in Theorem~\ref{thm:main}} \label{apd:a-thm-main}

If $\rho_\beta =0$, we have $\beta =0$ almost surely. Then, from \eqref{eq:beta},  we get
\begin{equation*}
	\left\| \frac{1}{m} \sum_{j=1}^m \hat{\beta}(S_j) \right\|^2  = \frac{1}{m}  |(e^m)^\top v |^2 \|\Db_x^{-1} b\|^2, \ \ \ 
	\frac{1}{m} \sum_{j=1}^{m}\|\hat{\beta}(S_j)\|^2 = \frac{1}{m} \|v\|^2 \|\Db_x^{-1} b\|^2.
\end{equation*}
Since $\Sigmab$ has eigenvalue decomposition $\Pb \Lambda \Pb^\top$ such that $\Lambda = \diag(\lambda_1,\dots, \lambda_m)$, the above equations yields that
\begin{equation}
	\begin{split}
		1-\Vcal  
		=& \frac{|(e^m)^\top v|^2}{\|v\|^2} \\
		=& \frac{v^\top \Sigmab^{-1/2}\Sigmab^{1/2} e^m (e^m)^\top \Sigmab^{1/2}\Sigmab^{-1/2}v}{v^\top \Sigmab^{-1/2} \Sigmab \Sigmab^{-1/2} v} \\
		\sim & \frac{g^\top \Pb \Lambda^{1/2} \Pb^\top e^m (e^m)^\top \Pb \Lambda^{1/2} \Pb^\top g}{g^\top \Pb \Lambda \Pb^\top g} \\
		\sim & \frac{g^\top  \Lambda^{1/2} \Pb e^m (e^m)^\top \Pb \Lambda^{1/2} g}{g^\top \Lambda g} \\
		=& \frac{\left|\sum_{i=1}^m c_i \lambda_i^{1/2} g_i \right|^2}{\sum_{i=1}^m \lambda_i g_i^2},
	\end{split}
\end{equation}
where $g \sim \Ncal(0, \Ib)$ and we have used rotation invariance of $g$, and $c = (c_1,\dots, c_m) = (e^m)^\top \Pb$.


For $\rho_\beta > 0$, using \eqref{eq:beta}, first we can write
\begin{equation*}
	\begin{split}
		1-\Vcal
		=&  \frac{\left\|\frac{1}{ m} \sum_{j=1}^m \hat{\beta}(S_j) \right\|^2 }{\frac{1}{m} \sum_{j=1}^{m}\|\hat{\beta}(S_j)\|^2}  \\
		=& 1 -   \frac{\frac{1}{m} \sum_{j=1}^{m}\|\hat{\beta}(S_j)\|^2 - \left\|\frac{1}{ m} \sum_{j=1}^m \hat{\beta}(S_j) \right\|^2 }{\frac{1}{m} \sum_{j=1}^{m}\|\hat{\beta}(S_j)\|^2} \\
		=& 1 -  \frac{\frac{1}{ m} \|v\|^2 \|\Db_x^{-1} b\|^2 -  \frac{1}{m}  |(e^m)^\top v |^2 \|\Db_x^{-1} b\|^2}{\|\beta\|^2 +  \frac{2}{\sqrt{m}} (e^m)^\top v \beta^\top\Db_x^{-1} b + \frac{1}{ m} \|v\|^2 \|\Db_x^{-1} b\|^2}   \\
		=& 1 -  \frac{\frac{1}{ m\rho_\beta^2} \|v\|^2 \|\Db_x^{-1} b\|^2 -  \frac{1}{m\rho_\beta^2}  |(e^m)^\top v |^2 \|\Db_x^{-1} b\|^2}{\frac{\|\beta\|^2}{\rho_\beta^2} +  \frac{2}{\sqrt{m}\rho_\beta^2} (e^m)^\top v \beta^\top\Db_x^{-1} b + \frac{1}{ m \rho_\beta^2} \|v\|^2 \|\Db_x^{-1} b\|^2}.
	\end{split}
\end{equation*}
Combining \eqref{eq:emv-mean} and \eqref{eq:emv-var} implies
\begin{equation*}
	\EE \left[ \frac{1}{ m\rho_\beta^2} \|v\|^2 \|\Db_x^{-1} b\|^2 \right] \rightarrow 0, \ \ \ \Var \left[ \frac{1}{ m\rho_\beta^2} \|v\|^2 \|\Db_x^{-1} b\|^2 \right] \rightarrow 0,
\end{equation*}
and
\begin{equation*}
	\EE \left[ \frac{1}{m\rho_\beta^2}  |(e^m)^\top v |^2 \|\Db_x^{-1} b\|^2 \right] \rightarrow 0, \ \ \ \Var \left[ \frac{1}{m\rho_\beta^2}  |(e^m)^\top v |^2 \|\Db_x^{-1} b\|^2\right] \rightarrow 0,
\end{equation*}
as $\rho_\gamma^2/\rho_\beta^2 \rightarrow 0$, which also yields $\frac{2}{\sqrt{m}\rho_\beta^2} (e^m)^\top v \beta^\top\Db_x^{-1} b \stackrel{p}{\rightarrow} 0$ as $\rho_\gamma^2/\rho_\beta^2 \rightarrow 0$.

Then, continuous mapping theorem~\cite{durrett2019probability} implies
\begin{equation*}
	1-\frac{\frac{1}{ m\rho_\beta^2} \|v\|^2 \|\Db_x^{-1} b\|^2 -  \frac{1}{m\rho_\beta^2}  |(e^m)^\top v |^2 \|\Db_x^{-1} b\|^2}{\frac{\|\beta\|^2}{\rho_\beta^2} +  \frac{2}{\sqrt{m}\rho_\beta^2} (e^m)^\top v \beta^\top\Db_x^{-1} b + \frac{1}{ m \rho_\beta^2} \|v\|^2 \|\Db_x^{-1} b\|^2}
	\stackrel{p}{\rightarrow}
	1.
\end{equation*}
Moreover, we know
\begin{equation*}
	0\leq \frac{\left\|\frac{1}{ m} \sum_{j=1}^m \hat{\beta}(S_j) \right\|^2 }{\frac{1}{m} \sum_{j=1}^{m}\|\hat{\beta}(S_j)\|^2} \leq 1,
\end{equation*}
so $\{1-\Vcal(\hat{\beta}(S_1), \dots, \hat{\beta}(S_m)): \rho_\alpha/\rho_\beta \geq0\}$ is uniformly integrable and $\EE[1-\Vcal] \rightarrow 1$ as $\rho_\alpha/\rho_\beta \rightarrow \infty$.
As a result, the first part of Theorem~\ref{thm:main} follows.

\subsection{Proof (b) in Theorem~\ref{thm:main}} \label{apd:b-thm-main}

If $\tr(\Sigmab)/m = o(q) $, we have
\begin{equation*}
	0< (e^m)^\top \Sigmab e < \frac{1}{m}\max \left\lbrace \tr(\Sigmab), \sqrt{\tr(\Sigmab^2)} \right\rbrace = o(q),
\end{equation*}
which implies
\begin{equation*}
	\frac{1}{m}  |(e^m)^\top v |^2 \rightarrow 0 , \ \ \ \frac{1}{ m} \|v\|^2 \rightarrow 0,
\end{equation*}
in probability. Moreover, it holds that
\begin{equation*}
	\Var \left[ \frac{2}{\sqrt{m}} (e^m)^\top v \beta^\top\Db_x^{-1} b \right] =   \frac{4 \rho_\gamma^2}{m d q}  (e^m)^\top \Sigmab e \|D_x^{-1} b\|^2 \rightarrow 0,
\end{equation*}
yielding
\begin{equation*}
	\frac{2}{\sqrt{m}} (e^m)^\top v \beta^\top\Db_x^{-1} b \rightarrow 0,
\end{equation*}
in probability. Consequently, we get
\begin{equation*}
	\frac{1}{m}\sum_{j=1}^{m} \hat{\beta}(S_j) = \beta + \frac{1}{ \sqrt{m}} (e^m)^\top v \Db_x^{-1} b \rightarrow \beta,
\end{equation*}
in probability. Furthermore, if $\rho_\beta > 0$ and $q \rightarrow \infty$, we have
\begin{equation}
	1-\Vcal  \rightarrow \frac{\|\beta\|^2}{\|\beta\|^2} = 1,
\end{equation}
in probability by Theorem 2.7 in \citet{van2000asymptotic}. On the other hand, if $\beta = 0$, we have
\begin{equation}
	1-\Vcal = \frac{\left\| \frac{1}{m} \sum_{j=1}^m \hat{\beta}(S_j) \right\|^2}{\frac{1}{m} \sum_{j=1}^{m}  \|\hat{\beta}(S_j)\|^2} \rightarrow \frac{(e^m)^\top \Sigmab e}{\tr(\Sigmab)},
\end{equation}
in probability as $q \rightarrow \infty$. Then this theorem follows.

\subsection{More details of Theorem~\ref{thm:main}}\label{apd:example}

Two propositions are given to illustrate Theorem~\ref{thm:main}. The latter provides an example where our method can identify the cause when there is a hidden confounder, while the partial correlation has no information for finding causal drivers. However, note that the partial correlation is not designed for detecting causes and can be used without any strong assumption on confounding.

\begin{proposition}\label{prop:sufficient-condition}
	Suppose $\Cov (N_w) = \Ib$ and we have $m$ subsets $S_1, \dots, S_m \subset \lbrace 1,\dots, q \rbrace$ such that
	\begin{equation}\label{eq:subset}
		\bigcup_{j=1}^m S_j = \lbrace 1, \dots, q\rbrace, \ \ \ S_i \bigcap S_j = \emptyset \text{ for $i\neq j$}.
	\end{equation}
	If $\|a\|^2 = o(q)$, the condition in Theorem~\ref{thm:main} holds, i.e.
	\begin{equation}\label{eq:cond}
		\frac{\tr(\Sigmab)}{m}  = o(q), \text{  }
	\end{equation}
	as $q\rightarrow\infty$.
\end{proposition}

\begin{proof}[Proof of Proposition~\ref{prop:sufficient-condition}]
	we know
	\begin{equation}
		\frac{1}{m} \tr(\Sigmab) = \frac{1}{m} \sum_{j=1}^m \frac{\|a_{S_j^c}\|^2}{(1+\|b\|_x^2+\|a_{S_j}\|^2)^2}\leq \frac{(m-1)\|a\|^2}{m(1+\|b\|_x^2)^2}
	\end{equation}
	where $\Sigmab$ is defined in \eqref{eq:sigmab}. Thus, $\|a\|^2=o(q)$ implies $ \tr(\Sigmab)/m = o(q)$, which completes the proof.
\end{proof}

\begin{proposition}\label{prop:partial-corr}
	Assume $d=1=b, \rho_\beta=0$, $\Cov (N_w) = \Ib$, and $W_1$ is the unobserved confounder with the non-trivial coefficient $a_1 \neq 0$.
	Then the partial correlation does not converge to a constant even when $q \rightarrow \infty$.
	On the other hand, suppose that we have subsets $S_1, \dots, S_m \subset \lbrace 2,\dots, q \rbrace$ satisfying \eqref{eq:subset}. If $\|a\|^2 = o(q)$, the condition \eqref{eq:cond} in Theorem~\ref{thm:main} holds.
\end{proposition}

\begin{proof}[Proof of Proposition~\ref{prop:partial-corr}]
	Define $W_{-1}=(W_2, \dots, W_q)^\top$, $a_{-1}=(a_2,\dots, a_q)^\top$, and $\gamma_{-1} = (\gamma_2,\dots, \gamma_q)^\top$.
	Then the residuals are
	\begin{equation*}
		\begin{split}
			R_Y
			:=& Y-\EE[Y|W_{-1}, \gamma]  =   \frac{\gamma_1 a_1 Z}{1+\|a_{-1}\|} + N_y - \frac{ \gamma_1 a_1 a_{-1}^\top  N_{w_{-1}}  }{1+\|a_{-1}\|} \\
			R_X
			:=& X-\EE[X|W_{-1}, \gamma] = \frac{Z}{1+\|a_{-1}\|^2} + N_x - \frac{a_{-1}^\top N_{w_{-1}}}{1+\|a_{-1}\|^2}
		\end{split}
	\end{equation*}
	Therefore, we have
	\begin{equation*}
		\EE [R_Y R_X|\gamma] = a_1\gamma_1
	\end{equation*}
	which implies
	\begin{equation*}
		\frac{ \EE[R_Y R_X | \gamma ]}{\Var[R_Y | \gamma] \Var[R_X| \gamma]} \neq 0
	\end{equation*}
	
	On the other hand, we know $m = q/k$ and
	\begin{equation*}
		\frac{1}{m} \tr(\Sigmab) \leq \frac{(m-1)\|a\|^2 + a_1^2}{m(1+\|b\|_x^2)^2}.
	\end{equation*}
	Thus, $\|a\|^2 = o(q)$ implies $ \tr(\Sigmab)/m = o(q)$, which completes the proof.
\end{proof}

\section{Rotation invariant distributions} \label{apd:rotation-invariant-distribution}

In this section, we generalize our result in Theorem~\ref{thm:main} to $r>1$ and rotation invariant prior. Recall the definition of rotation invariant distributions:
\begin{definition} \label{def:spherically-symmetric}
	A random vector $\nu=(\nu_1, \dots, \nu_q)^\top$ is spherically symmetric if the distribution of every linear form is the same as the distribution of $\nu_1$, i.e.$e^\top \nu \sim \nu_1$ for all $e$ provided that $\|e\|=1$.
\end{definition}
From the definition, we get two important properties.
\begin{lemma} \label{lemma:sherically-symmetric}
	Suppose that $\nu$ is spherically symmetric. Then, it holds that
	\begin{enumerate}
		\item $\Ob \nu \sim \nu$, where $\Ob$ is a orthogonal matrix,
		\item $\nu = \tau U$, where $U$ is uniformly distributed on the unit sphere in $\RR^d$, $\tau \geq 0$ is a real value with distribution $\tau \sim \|\nu\|$, and $\tau, U$ are independent.
	\end{enumerate}
\end{lemma}
\begin{proof}
	See Theorem 4.1.2 in \cite{bryc2012normal}.
\end{proof}
From the first property, we can see that the spherically symmetric random variable is rotation invariant. Moreover, the second property states rotation invariant distributions can be decomposited by a uniform distribution on the unit sphere and a positive random variable. Lemma~\ref{lemma:sherically-symmetric} is so powerful since we know the exact marginal distribution of uniformly distributed on the unit sphere, which facilitates many calculations and is presented as follows:
\begin{proposition} \label{prop:uniformly-dist}
	Suppose  $U$ is uniformly distributed on the unit sphere in $\RR^d$. Then, it holds that $e^\top U \sim \|e\|(2-\operatorname{Beta}(\frac{d-1}{2}, \frac{d-1}{2}))$, where $\operatorname{Beta}$ is beta distribution.
\end{proposition}
\begin{proof}
	Since $U$ is spherically symmetric, we get $e^\top U \sim \|e\| U_1$.
	From example 4.1.2 in \cite{bryc2012normal}, we know the density of the marginal distribution of $U_1$ is
	\begin{equation}
		f_{U_1}(u) \propto (1-u^2)^{\frac{d-1}{2}-1} = (1-u)^{\frac{d-1}{2}-1} (1+u)^{\frac{d-1}{2}-1} \text{ for $-1\leq u\leq 1$}.
	\end{equation}
	On the other hand, the density of $2\operatorname{Beta}(\frac{d-1}{2}, \frac{d-1}{2})-1$ is 
	\begin{equation}
		\begin{split}
			f_{2\operatorname{Beta}(\frac{d-1}{2}, \frac{d-1}{2})-1} (u) 
			&\propto f_{\operatorname{Beta}(\frac{d-1}{2}, \frac{d-1}{2})} \left( \frac{u+1}{2} \right) \\
			& \propto \left( \frac{u+1}{2} \right)^{\frac{d-1}{2}-1} \left( 1-\frac{u+1}{2} \right)^{\frac{d-1}{2}-1} \\
			&\propto (1-u)^{\frac{d-1}{2}-1} (1+u)^{\frac{d-1}{2}-1}  \text{ for $-1\leq u\leq 1$}.
		\end{split}
	\end{equation}
	The proof is completed by comparing two densities.
\end{proof}
The next key result shows upper bounds of the second moment of quadratic forms for rotation invariant random variables.
\begin{proposition} \label{prop:sherically-symmetric}
	Suppose that $\nu$ is spherically symmetric. Then, it holds that $\EE (\nu^\top H \nu)^2 \leq  (\tr H)^2 \EE \|\nu\|^4$.
\end{proposition}
\begin{proof}
	Let the eigendecomposition of $\Hb$ is $\Qb \diag(\lambda_H) \Qb^\top$. 
	Then we get
	\begin{equation}
		\begin{split}
			\EE (\nu^\top H \nu)^2 
			&= \EE (\nu^\top \Qb \diag(\lambda_H) \Qb^\top \nu)^2 \\
			&\stackrel{\text{1 in Lemma~\ref{lemma:sherically-symmetric}}}{=} \EE (\nu^\top \diag(\lambda_H) \nu)^2 \\
			&= \EE (\|\diag(\lambda_H)^{1/2} \nu\|^2)^2 \\
			&\stackrel{\text{2 in Lemma~\ref{lemma:sherically-symmetric}}}{=} \EE (\|\diag(\lambda_H)^{1/2} U\|^2)^2 \EE \tau^4 \\
			&\stackrel{\text{2 in Lemma~\ref{lemma:sherically-symmetric}}}{=} \EE (\|\diag(\lambda_H)^{1/2} U\|^2)^2 \EE \|\nu\|^4 \\
			& = \EE\left[\sum_{i=1}^d (\lambda_H)_i U_i^2 \right]^2 \EE \|\nu\|^4  \\
			&\stackrel{|U_i|\leq 1}{=} (\tr H)^2 \EE \|\nu\|^4.
		\end{split}
	\end{equation}
	Therefore, this proposition follows.
\end{proof}
Now we are ready to state our generalized results which includes Theorem~\ref{thm:main-res}.
\begin{theorem}\label{thm:main-2}
	Let $S_1, \dots , S_m \subsetneq \{1,\dots , q\}$. Assume that  $\gamma$ is a spherically symmetric random vector such that $\EE\|\gamma\|^4 = O(\rho_\gamma^4/q^2)$. Define $\Cb_m = \frac{1}{m} \sum_{j=1}^{m}\Cb(S_j)$ and $\widetilde{\Cb}_m = \frac{1}{ m} \sum_{j=1}^{m} \Cb(S_j)^\top \Cb(S_j) $.
	
	\noindent
	(a) We have two cases: if $\rho_\beta = 0$, it holds that
	\begin{equation*}
		\Vcal\left(\hat{\beta}(S_1), \dots, \hat{\beta}(S_m)\right) \sim 1 -\frac{\gamma^\top \Cb_m^\top \Cb_m \gamma}{\gamma^\top \widetilde{\Cb}_m \gamma}.
	\end{equation*}
	If $\rho_\gamma^2/\rho_\beta^2 \rightarrow 0$, we have $  \EE\Vcal(\hat{\beta}(S_1), \dots, \hat{\beta}(S_m)) = o(1)$.
	
	\noindent
	(b) Further, assume that 
	\begin{equation} \label{eq:confoundingSte-2}
		\tr (\widetilde{\Cb}_m)   = o(q), \ \ \ \text{ as $q\rightarrow\infty$, }
	\end{equation}
	then,
	\begin{equation} \label{eq:c}
		\Vcal(\hat{\beta}(S_1), \dots, \hat{\beta}(S_m)) \ \stackrel{p}{\rightarrow}
		\begin{cases}
			0 &\text{ if $\rho_\beta > 0$ } \\
			1 - \frac{\tr (\Cb_m^\top \Cb_m)}{\tr (\widetilde{\Cb}_m)}  & \text{ if $\rho_\beta = 0$ }
		\end{cases},
	\end{equation}
	and $\frac{1}{m}\sum_{j=1}^{m} \hat{\beta}(S_j) \stackrel{p}{\rightarrow} \beta$, where $\rho_\beta =\|\beta\|$.
\end{theorem}
We can see that Theorem~\ref{thm:main} is actually a special case of Theorem~\ref{thm:main-2}. Again, the condition that $\beta$ is drawn from normal distributions is not necessary to relate stable regression coefficients to causal relations and id imposed due to technical consideration. The assumption \eqref{eq:confoundingSte-2} is similar to \eqref{eq:confoundingSte} and means mild correlation among $\{\hat{\beta}(S_j)\}_{j=1}^m$ caused by confounders between $X$ and $W$.

To prove Theorem~\ref{thm:main-2}, we need some calculations.
Recall the regression coefficients $\hat{\beta}(S)$ on $X, W_S$ is
\begin{equation} \label{eq:regression-coef-2}
	\begin{split}
		\hat{\beta}(S)
		&= (\Ib_d, \zero) \Var \left(X, W_S\right)^{-1}
		\Cov \left( (X, W_S), Y \right) \\
		&= \beta + \Cb(S)  \gamma
	\end{split}
\end{equation}
where$\Db_{S} = \diag (\sigma_{w_S})$ and
\begin{equation}
	(\Cb(S)^\top)_S = (\Db_x^{-1} \Bb (\Ib+\Bb^\top \Db_x^{-1} \Bb + \Ab_{S}^\top \Db_{S} \Ab_{S})^{-1} \Ab_{S^c}^\top)^\top, \text{ and } (\Cb(S)^\top)_{S^c} = \zero.
\end{equation} 
Then we have
\begin{equation} \label{eq:beta-2}
	\begin{split}
		\left\| \frac{1}{ m} \sum_{j=1}^m \hat{\beta}(S_j) \right\|^2 =  \left\| \beta + \frac{1}{m} \sum_{j=1}^m \Cb(S_j) \gamma \right\|^2 =\|\beta\|^2+ 2 \beta^\top \Cb_m \gamma +  \gamma^\top \Cb_m^\top \Cb_m \gamma, \\
		\frac{1}{m} \sum_{j=1}^{m}\|\hat{\beta}(S_j)\|^2 = \frac{1}{m} \sum_{j=1}^{m}\|  \beta + \Cb(S_j) \gamma \|^2 = \|\beta\|^2 +  2 \beta^\top \Cb_m \gamma + \gamma^\top \widetilde{\Cb}_m \gamma
	\end{split}
\end{equation}
where $\Cb_m =  \frac{1}{m}\sum_{j=1}^m  \Cb(S_j)$ and $\widetilde{\Cb}_m = \frac{1}{ m} \sum_{j=1}^{m} \Cb(S_j)^\top \Cb(S_j)$.

\subsection{Proof of (a) in Theorem~\ref{thm:main-2}}

If $\rho_\beta =0$, we have $\beta =0$ almost surely. Then, from \eqref{eq:beta-2},  we get
\begin{equation}
	\begin{split}
		1-\Vcal = \frac{\gamma^\top \Cb_m^\top \Cb_m \gamma}{\gamma^\top  \widetilde{\Cb}_m \gamma}.
	\end{split}
\end{equation}
For $\rho_\beta > 0$ and $\rho_\gamma^2/\rho_\beta^2 \rightarrow 0$, following the same argument in Appendix~\ref{apd:a-thm-main}, it suffices to show
\begin{equation}
	\begin{split}
		\EE \left( \frac{1}{\rho_\beta^2} \gamma^\top \Cb_m^\top \Cb_m \gamma \right)^2 \rightarrow 0, \ \ \
		\EE \left( \frac{1}{\rho_\beta^2} \gamma^\top \widetilde{\Cb}_m \gamma\right)^2 \rightarrow 0,
	\end{split}
\end{equation}
since $\EE(\gamma^\top \Cb_m^\top \Cb_m \gamma)^2 \rightarrow 0$ implies $\EE (2 \beta^\top \Cb_m \gamma)^2 \rightarrow 0$.

By Proposition~\ref{prop:sherically-symmetric} and assumptions, we have

\begin{subequations} \label{eq:second-moment}
	\begin{align}
		\EE \left( \frac{1}{\rho_\beta^2} \gamma^\top \Cb_m^\top \Cb_m \gamma \right)^2 &\leq \frac{ [\tr(\Cb_m^\top \Cb_m)]^2 \EE \|\nu\|^4}{\rho_\beta^4} = O\left(\frac{\rho_\gamma^4  [\tr(\Cb_m^\top \Cb_m)]^2}{\rho_\beta^4 q^2}\right), \label{eq:1} \\
		\EE \left( \frac{1}{\rho_\beta^2} \gamma^\top \Cb_m^\top \Cb_m \gamma \right)^2 &\leq \frac{ [\tr(\widetilde{\Cb}_m)]^2 \EE \|\nu\|^4}{\rho_\beta^4} = O\left(\frac{\rho_\gamma^4 [\tr(\widetilde{\Cb}_m)]^2}{\rho_\beta^4 q^2}\right) \label{eq:2}.
	\end{align}
\end{subequations}

Both of them converges to zero as $\rho_\gamma/\rho_\beta \rightarrow 0$.
As a result, the first part of Theorem~\ref{thm:main-2} follows.

\subsection{Proof (b) in Theorem~\ref{thm:main-2}}

First, we know two facts:
\begin{enumerate}
	\item $\EE \gamma^\top H \gamma = \tr (H \EE(\gamma \gamma^\top))$.
	\item If $U$ is uniformly distributed on the unit sphere in $\RR^q$, then $\EE U U^\top = \frac{1}{q}\Ib$.
\end{enumerate}

Then by 2. in Lemma~\ref{lemma:sherically-symmetric}, we get
\begin{equation*}
	\begin{split}
		\EE \left( \gamma^\top \Cb_m^\top \Cb_m \gamma \right) &= \EE \|\gamma\|^2 \tr (\Cb_m^\top \Cb_m)/q, \\
		\EE \left( \gamma^\top \widetilde{\Cb}_m \gamma \right) &= \EE \|\gamma\|^2 \tr (\widetilde{\Cb}_m)/q
	\end{split}
\end{equation*}

Moreover, the assumption \eqref{eq:confoundingSte-2} implies \eqref{eq:1} and \eqref{eq:2} converge to zero. Consequently, following the same argument in Appendix~\ref{apd:b-thm-main}, we get the following results:
If $\rho_\beta > 0$ and $q \rightarrow \infty$, it holds that
\begin{equation}
	1-\Vcal  \stackrel{p}{\rightarrow} \frac{\|\beta\|^2}{\|\beta\|^2} = 1,
\end{equation}
while if $\beta = 0$, we have
\begin{equation}
	1-\Vcal = \frac{\left\| \frac{1}{m} \sum_{j=1}^m \hat{\beta}(S_j) \right\|^2}{\frac{1}{m} \sum_{j=1}^{m}  \|\hat{\beta}(S_j)\|^2} \stackrel{p}{\rightarrow} \frac{\tr (\Cb_m^\top \Cb_m)}{\tr (\widetilde{\Cb}_m)},
\end{equation}
as $q \rightarrow \infty$. Then this theorem follows.

\section{Proof of Theorem~\ref{cor:gaussian-case}}

Drawing $M$ permutations $\{\pi_i\}_{i=1}^M$ uniformly, we know that
\begin{equation*}
	((\bar{\Yb}(\pi_1), \Xb, \Wb), \dots, (\bar{\Yb}(\pi_M), \Xb, \Wb))
\end{equation*}
is exchangeable under the null hypothesis $H_0$ in \eqref{eq:hypothesis} where $\bar{\Yb}^{\pi_i} = \Wb \gamma  + \Nb_y^{\pi_i}$ since $\Nb_y$ is an independent noise.
Given any valid statistic $T(\Yb, \Xb, \Wb)$ such as $T = \Vcal$ in our case, $p$-values are of the form
\begin{equation*}
	\begin{split}
		& p((\Yb, \Xb, \Wb), (\bar{\Yb}(\pi_1), \Xb, \Wb), \dots,(\bar{\Yb}(\pi_M), \Xb, \Wb)) \\
		=& \frac{\sum_{i=1}^{M} \ind \lbrace T(\bar{\Yb}^{\pi_i}, \Xb, \Wb) \geq T(\Yb, \Xb, \Wb) \rbrace + 1}{M+1}.
	\end{split}
\end{equation*}

By exchangeability, we get
\begin{equation}
	\PP(p((\Yb, \Xb, \Wb), (\bar{\Yb}(\pi_1), \Xb, \Wb), \dots,(\bar{\Yb}(\pi_M), \Xb, \Wb)) \leq \alpha \mid (\Xb, \Wb), \gamma) \leq \alpha
\end{equation}

For the permutation $\pi_i$ and the corresponding permutation matrix $P_i$ and an estimation $\hat{\gamma}$, we get
\begin{equation}
	\begin{split}
		\hat{\Yb}(\pi_i) \mid H_0, \Xb, \Wb, \gamma
		=& \Wb \hat{\gamma} + P_i(\Yb-\Wb\hat{\gamma}) \mid H_0, \Xb, \Wb, \gamma\\
		=& \Wb \hat{\gamma} + P_i\Wb(\gamma - \hat{\gamma}) + P_i \Nb_y \mid H_0, \Xb, \Wb, \gamma \\
		\sim & \Ncal(\Wb \hat{\gamma} + P_i\Wb(\gamma-\hat{\gamma}), \sigma_y^2 \Ib)
	\end{split}
\end{equation}

Recall that $e^M=(1/\sqrt{M},\dots, 1/\sqrt{M})^\top$ is a $M$-dimensional unit vector. Since $P_i$ and $P_j$ are drawn uniformly and independently, it holds that
\begin{equation*}
	\Cov( \bar{\Yb}^{\pi_{i}}, \bar{\Yb}^{\pi_{j}} \mid H_0, \Xb, \Wb, \gamma) = \EE [P_i \Nb_{y} \Nb_{y}^\top P_j ]= \sigma_y^2 e_M e_M^\top ,
\end{equation*}
and
\begin{equation*}
	\Cov( \hat{\Yb}^{\pi_{i}}, \hat{\Yb}^{\pi_{j}} \mid H_0, \Xb, \Wb, \gamma) = \EE [P_i \Nb_{y} \Nb_{y}^\top P_j ]= \sigma_y^2 e_M e_M^\top .
\end{equation*}
Therefore, we know eigenvalues of the covariance matrices $ (\bar{\Yb}(\pi_1), \dots, \bar{\Yb}(\pi_M))$ and $(\hat{\Yb}(\pi_1), \dots, \hat{\Yb}(\pi_M))$ conditional on $H_0, \Xb, \Wb, \gamma$ are $(M\sigma_y^2, \sigma_y^2, \dots,  \sigma_y^2)$.
It is well-known that if $g_1 \sim \Ncal(\mu_1, G) $ and $g_2 \sim \Ncal(\mu_2, G)$, then
\begin{equation*}
	d_{\text{KL}}(g_1, g_2) = \frac{1}{2} (\mu_1 - \mu_2)^\top G^{-1} (\mu_1 - \mu_2).
\end{equation*}
Then we get
\begin{equation*}
	\begin{split}
		& d_{\text{KL}}((\bar{\Yb}(\pi_1), \dots, \bar{\Yb}(\pi_M)),  (\hat{\Yb}(\pi_1), \dots, \hat{\Yb}(\pi_M))) \\
		\leq& \frac{1}{2 \sigma_y^2} \sum_{i=1}^{M}\|\Wb (\gamma -\hat{\gamma}) + P_i(\Wb(\gamma- \hat{\gamma}))\|^2 \\
		\leq& \frac{M \|\Wb(\gamma- \hat{\gamma})\|^2}{\sigma_y^2},
	\end{split}
\end{equation*}
where $d_{\text{KL}}$ is the Kullback–Leibler divergence.

Note that the total variation distance $d_{\text{TV}}$ is defined as $d_{\text{TV}}(Q_1, Q_2) = \sup_E |Q_1(E) - Q_2(E) |$ where the supremum is taken over all measurable sets.
Then by the definition of total variation, it follows that
\begin{equation*}
	\begin{split}
		& \PP(p((\Yb, \Xb, \Wb), (\hat{\Yb}(\pi_1), \Xb, \Wb), \dots,(\hat{\Yb}(\pi_M), \Xb, \Wb)) \leq \alpha \mid \Xb, \Wb, \gamma) \\
		\leq& \PP(p((\Yb, \Xb, \Wb), (\bar{\Yb}(\pi_1), \Xb, \Wb), \dots,(\bar{\Yb}(\pi_M), \Xb, \Wb)) \leq \alpha \mid \Xb, \Wb, \gamma)  \\
		&+  d_{\text{TV}} \lbrace [(\bar{\Yb}(\pi_1), \dots, \bar{\Yb}(\pi_M))\mid \Xb, \Wb, \gamma], [(\hat{\Yb}(\pi_1), \dots, \hat{\Yb}(\pi_M))\mid \Xb, \Wb, \gamma] \rbrace.
	\end{split}
\end{equation*}
By Pinsker’s inequality relating total variation distance to the Kullback Leibler divergence $d_{\text{KL}}$, we get
\begin{equation*}
	\begin{split}
		& d_{\text{TV}}^2 \left\lbrace [(\bar{\Yb}(\pi_1), \dots, \bar{\Yb}(\pi_M))\mid \Xb, \Wb, \gamma], [(\hat{\Yb}(\pi_1), \dots, \hat{\Yb}(\pi_M))\mid \Xb, \Wb, \gamma, ] \right\rbrace \\
		\leq & \frac{1}{2} d_{\text{KL}} \left\lbrace [(\bar{\Yb}(\pi_1), \dots, \bar{\Yb}(\pi_M))\mid \Xb, \Wb, \gamma], [(\hat{\Yb}(\pi_1), \dots, \hat{\Yb}(\pi_M))\mid \Xb, \Wb, \gamma] \right\rbrace \\
		\leq & \frac{ M \|\Wb (\gamma -\hat{\gamma})\|^2}{2 \sigma_y^2},
	\end{split}
\end{equation*}
which completes the proof.

\section{Additional Experiments}\label{apd:additional-experiments}

This section includes additional experiments.

{\bf Results of varying parameters: }
First, we vary parameters $\rho_\beta, d$, and $|S_j|$ to examine more details of our approach. Note that the basic setting is that $d=3$, $r=5$, $\ell=300$, $\rho_\gamma =2.5$ for {\it data generating procedure} and $m =40, M=400, k=3$ for {\it Algorithm~\ref{alg:PT}+\ref{alg:model-averaging}+RS}. We repeated the experiment 500 times for each point in Figure ~\ref{fig:sem2}.

\begin{figure*}[h]
	\centering
	\caption{The result of varying parameters.} \label{fig:sem2}
	\includegraphics[width=0.32\textwidth]{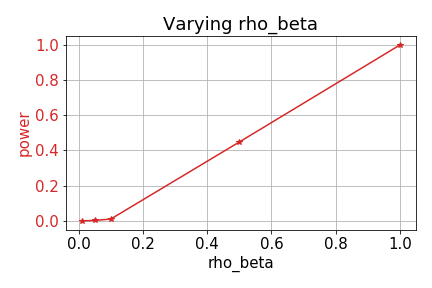}
	\includegraphics[width=0.32\textwidth]{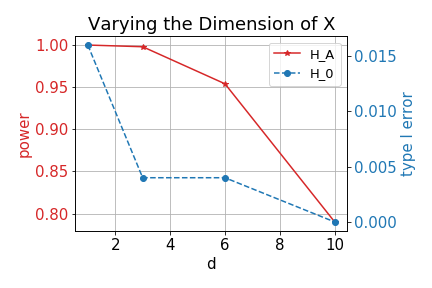}
	\includegraphics[width=0.32\textwidth]{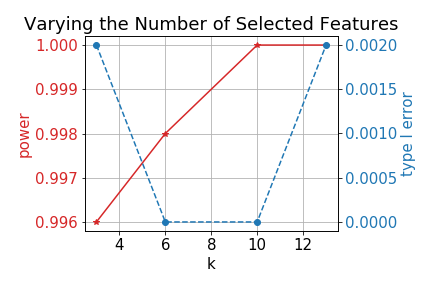}
\end{figure*}

We can see the power increasing exponentially with respect to $\rho_\beta$ and increasing the dimension of $X$ decreases the power. This is only slightly changed when varying $k$. Type I error is small in all settings.

{\bf Results of varying numbers of hidden components: } Next, we conduct experiments with varying numbers of hidden components for the various competing approaches as well as our random selection method. We used $d=1$, $q=500$, $r=5$, $\ell=100$, $\rho_\gamma =1$, $m = 100, M =1000$, $|S_j|=10$, and generated data for $\rho_\beta = 0, 1$.  
We consider two ways to generate $\gamma$:

\noindent
{\it Generating model 1:} Draw $\gamma^\prime$ from $\Ncal(0,\Ib)$ and $\gamma= \rho_\alpha \gamma^\prime/\|\gamma^\prime\|$,

\noindent
{\it Generating model 2:} Draw $\gamma^\prime$ from Student's t-distribution with degree of freedom $2.2$ and $\gamma= \rho_\alpha \gamma^\prime/\|\gamma^\prime\|$,

and three settings:

\noindent
{\it Setting 1:} $(W_1, \dots, W_{500})$ is accessible,

\noindent
{\it Setting 2:} We set $(W_{1}, \dots, W_{250})$ as latent variables. That is, only $(W_{251}, \dots, W_{500})$ is accessible,

\noindent
{\it Setting 3:} We set $(W_1, \dots, W_{450})$ as latent variables. That is, only $(W_{451}, \dots, W_{500})$ is accessible.

Note that in generating procedure 1, $\gamma^\prime$ is drawn from a rotation invariant and light-tailed distribution, i.e., a normal distribution. Thus, with high probability, the magnitude of all elements of $\gamma$ are similar, and there is no outlier. On the other hand, in generating procedure 2, $\gamma^\prime$ is drawn from a rotation invariant and heavy-tailed distribution, i.e., a student's t distribution with a low degree of freedom. Hence, with high probability,  the magnitude of elements of $\gamma$ are dominated by some outliers. In other words, most coefficients of $\gamma$ are very close to zero, while only a few elements' magnitudes are large.

The results are summarized in Table~\ref{table:HA} and Figure~\ref{fig:h0}. We first discuss the result for the alternative hypothesis.

\begin{table}[h]
	\centering
	\caption{This table shows the proportion p-values that are less than the significant level $\alpha$ under the alternative hypothesis~\eqref{eq:hypothesis} for $\rho_\beta=1$.}
	\begin{tabular}{cccccccccc}
		\toprule
		$H_A$ &    $\rho_\beta=1$       & \multicolumn{4}{c}{Generating model 1} & \multicolumn{4}{c}{Generating model 2} \\
		\cmidrule(lr){3-6} \cmidrule(lr){7-10}
		& $\alpha$ & \textbf{RS}        & BM       & FL       & DR       & \textbf{RS}       & BM        & FL        & DR       \\
		\cmidrule(lr){1-2} \cmidrule(lr){3-6} \cmidrule(lr){7-10}
		Setting 1 & 0.01     & {\bf 0.985}     & 0.98     & 0.3425   & 0.395    & {\bf 0.985}    & 0.9825    & 0.295     & 0.335    \\
		Setting 2 & 0.01     & {\bf 0.9775}    & 0.9625   & 0.8175   & 0.8775   & {\bf 0.98}     & 0.9725    & 0.8175    & 0.87     \\
		Setting 3 & 0.01     & 0.9675    & 0.945    & 0.995    & {\bf 0.9975}   & 0.955    & 0.9625    & 0.995     & {\bf 0.9975}   \\
		\bottomrule
	\end{tabular}
	\label{table:HA}
\end{table}

Table~\ref{table:HA} shows the similar pattern as the comparison result in Section~\ref{sec:simulations}.
We can see that RS shows the highest power in settings 1 and 2 for all generating models. As we mentioned in Section~\ref{sec:simulations}, it may be counter-intuitive that FL and DR work better as the number of hidden confounders increase, but since $\gamma$ is drawn from the prior with zero mean, the confounding effect will be offset when $q$ is large.
There is no significant difference between generating models 1 and 2.
Overall, RS is most stable and performing well in all settings.

\begin{figure*}[h]
	\centering
	\includegraphics[width=0.32\textwidth]{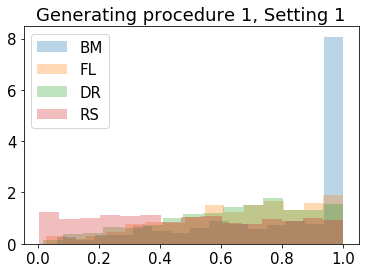}
	\includegraphics[width=0.32\textwidth]{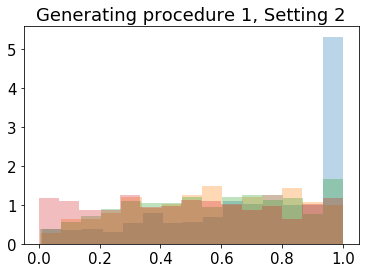}
	\includegraphics[width=0.32\textwidth]{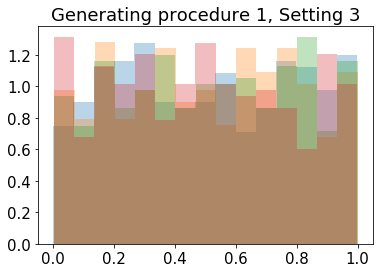}
	\includegraphics[width=0.32\textwidth]{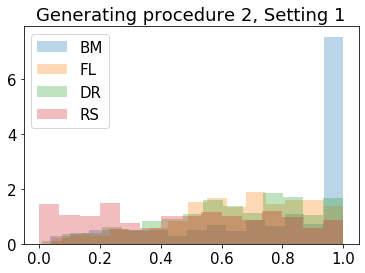}
	\includegraphics[width=0.32\textwidth]{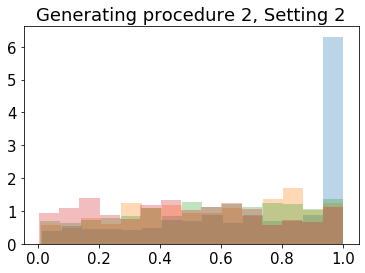}
	\includegraphics[width=0.32\textwidth]{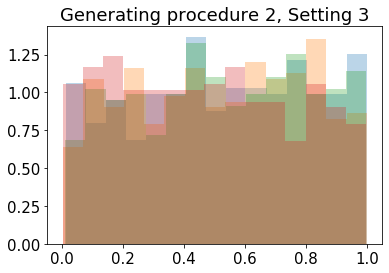}
	\caption{The histogram of p-values under the null hypothesis~\eqref{eq:hypothesis}.}
	\label{fig:h0}
\end{figure*}

From Figure~\ref{fig:h0}, we can see distributions of p-values of all methods are not uniform under the null hypothesis except our proposed method. Interestingly, increasing the number of hidden variables eliminates this phenomenon. On the other hand, the performance of RS and BM decreases, and the performance of FL and DR increases as the number of hidden variables increases. Again, this result is because $\gamma$ is drawn from some prior with zero mean, and the confounding effect is cancelled, which aligns with our conclusion in Section~\ref{sec:experiments}. The distribution of $\gamma$ has no significant effect on all methods. DR outperforms RS in setting 3. However, we would like to point out additional experiments are conducted in the case $\rho_\alpha / \rho_\beta =1$, which is a simpler case compared to $\rho_\alpha / \rho_\beta =20/3$ in Section~\ref{sec:experiments}. Moreover, our proposed method is most robust in all settings, generating procedures, and various model parameters.

\section{Additional Result for the real dataset}\label{apd:real-data}

In this section, we include the full analysis result for OLS in Section~\ref{sec:real-dataset}. The summary is stated in Table~\ref{tab:OLS}. We can see that even though R-squared is not high, the probability of F-statistics is very small, which indicates we reject the hypothesis that all the variables have zero regression coefficients. Moreover, there is no significant evidence that the error assumption is wrong, and the sample size is large enough, which validates the use of linear regression.
All covariates are statistically significant except "urban" and "cue80". The reason may be "urban" and "cue80" are highly correlated to other features like "stwmfg80".

\begin{table}[h]
	\caption{OLS Regression Results}\label{tab:OLS}
	\begin{center}
		\begin{tabular}{lclc}
			\toprule
			\textbf{Dep. Variable:}    & bytest        & \textbf{  R-squared:         } & 0.193     \\
			\textbf{Model:}            & OLS           & \textbf{  Adj. R-squared:    } & 0.191     \\
			\textbf{Method:}           & Least Squares & \textbf{  F-statistic:       } & 102.4     \\
			\textbf{No. Observations:} & 4739          & \textbf{  Prob (F-statistic):} & 5.83e-210 \\
			\textbf{Df Residuals:}     & 4727          & \textbf{  Log-Likelihood:    } & -16470.   \\
			\textbf{Df Model:}         & 11            & \textbf{  AIC:               } & 3.296e+04 \\
			\textbf{Covariance Type:}  & nonrobust     & \textbf{  BIC:               } & 3.304e+04 \\ \bottomrule
		\end{tabular}
	\end{center}
	\begin{center}
		\begin{tabular}{lcccccc}
			\toprule
			& \textbf{coef} & \textbf{std err} & \textbf{t} & \textbf{P$> |$t$|$} & \textbf{[0.025} & \textbf{0.975]}  \\
			\midrule
			\textbf{Intercept} &      49.0890  &        0.906     &    54.198  &         0.000        &       47.313    &       50.865     \\
			\textbf{female}    &      -1.0427  &        0.229     &    -4.544  &         0.000        &       -1.493    &       -0.593     \\
			\textbf{black}     &      -6.8793  &        0.330     &   -20.851  &         0.000        &       -7.526    &       -6.232     \\
			\textbf{hispanic}  &      -4.0855  &        0.309     &   -13.243  &         0.000        &       -4.690    &       -3.481     \\
			\textbf{dadcoll}   &       2.8857  &        0.327     &     8.837  &         0.000        &        2.246    &        3.526     \\
			\textbf{momcoll}   &       2.1927  &        0.369     &     5.938  &         0.000        &        1.469    &        2.917     \\
			\textbf{ownhome}   &       1.1346  &        0.303     &     3.746  &         0.000        &        0.541    &        1.728     \\
			\textbf{urban}     &      -0.3818  &        0.289     &    -1.321  &         0.187        &       -0.948    &        0.185     \\
			\textbf{cue80}     &       0.0222  &        0.045     &     0.491  &         0.624        &       -0.067    &        0.111     \\
			\textbf{stwmfg80}  &       0.2795  &        0.090     &     3.094  &         0.002        &        0.102    &        0.457     \\
			\textbf{dist}      &      -0.2674  &        0.055     &    -4.883  &         0.000        &       -0.375    &       -0.160     \\
			\textbf{incomehi}  &       0.7163  &        0.275     &     2.604  &         0.009        &        0.177    &        1.256     \\
			\bottomrule
		\end{tabular}
	\end{center}
	\begin{center}
		\begin{tabular}{lclc}
			\toprule
			\textbf{Omnibus:}       & 145.874 & \textbf{  Durbin-Watson:     } &    1.781  \\
			\textbf{Prob(Omnibus):} &   0.000 & \textbf{  Jarque-Bera (JB):  } &   70.318  \\
			\textbf{Skew:}          &  -0.047 & \textbf{  Prob(JB):          } & 5.38e-16  \\
			\textbf{Kurtosis:}      &   2.411 & \textbf{  Cond. No.          } &     102.  \\
			\bottomrule
		\end{tabular}
	\end{center}
\end{table}

\end{document}